%% file: main.tex
\documentclass{article}

\input{header.tex}

\title{Strategic Prediction with Latent Aggregative Games}

\author{%
  Vikas K.~Garg\\
  MIT\\
  \texttt{vgarg@csail.mit.edu} \\
  \And
  Tommi Jaakkola \\
  MIT \\
  \texttt{tommi@csail.mit.edu}\\
  }

\usepackage{microtype}
\usepackage{graphicx}
\usepackage{subcaption}
\usepackage{booktabs}
\usepackage{grffile}
\usepackage{etoolbox}
\makeatletter
\patchcmd\@combinedblfloats{\box\@outputbox}{\unvbox\@outputbox}{}{\errmessage{\noexpand patch failed}}
\makeatother

\usepackage[colorinlistoftodos]{todonotes}

\usepackage[utf8]{inputenc} 
\usepackage[T1]{fontenc}    
\usepackage{url}            
\usepackage{booktabs}       
\usepackage{amsfonts}       
\usepackage{nicefrac}       
\usepackage{microtype}      
\usepackage{amsmath, amsthm, amssymb}
\usepackage{algorithm}
\usepackage{algorithmic}
\usepackage{tikz}
\usetikzlibrary{matrix,chains,positioning,decorations.pathreplacing,arrows}
\usetikzlibrary{bayesnet}
\usepackage{graphicx}
\usepackage{subcaption}

\begin{document}

\maketitle
\input{abstract}

\input{introduction_new}
\input{setting}
\input{learning}
\input{transfer}
\input{fp}
\input{experiments2}
\clearpage
\clearpage
\bibliography{main}
\bibliographystyle{unsrt}
\clearpage
\appendix
\input{proofs}

\end{document}

%% file: header.tex
\usepackage[preprint, nonatbib]{neurips_2019}
\usepackage{collcell}
\usepackage{booktabs}
\usepackage{amssymb, amsmath, amsthm}
\usepackage{color}
\usepackage{floatrow}

\newtheorem{theorem}{Theorem}

\newfloatcommand{capbtabbox}{table}[][\FBwidth]
\usepackage{blindtext}

\usepackage[utf8]{inputenc} 
\usepackage[T1]{fontenc}    
\usepackage{hyperref}       
\usepackage{url}            
\usepackage{booktabs}       
\usepackage{amsfonts}       
\usepackage{nicefrac}       
\usepackage{microtype}      
\usepackage{tikz}
\usetikzlibrary{shapes.geometric, arrows, calc}
\usepackage[numbers]{natbib}

%% file: abstract.tex
\begin{abstract}
We introduce a new class of context dependent, incomplete information games to serve as structured prediction models for settings with significant strategic interactions. Our games map the input context to outcomes by first condensing the input into private player types that specify the utilities, weighted interactions, as well as the initial strategies for the players. The game is played over multiple rounds where players respond to weighted aggregates of their neighbors' strategies. The predicted output from the model is a mixed strategy profile (a near-Nash equilibrium) and each observation is thought of as a sample from this strategy profile. We introduce two new aggregator paradigms with provably convergent game dynamics, and characterize the conditions under which our games are identifiable from data. Our games can be parameterized in a transferable manner so that the sets of players can change from one game to another. We demonstrate empirically that our games as models can recover meaningful strategic interactions from real voting data.
\end{abstract}

%% file: introduction_new.tex
\section{Introduction}
Structured prediction methods \cite{LMP2001, TCK2003, TJHA2005, NL2011} typically operate on parametric scoring functions whose maximizing assignment is used as the predicted configuration. Since the parameters can be learned directly to maximize prediction accuracy, often via surrogate losses, the methods have been successful across areas \cite{CSYU2015, GRSY2015, LHG2016, CKMY2016, OBL2017, PS2018}.
However, not all structured observations can  be naturally modeled as extrema of scoring functions. For instance, votes on a bill in the US Congress pertain to actions or strategies adopted by individual senators following several rounds of negotiations with a subset of other senators. These votes do not generally correspond to any common scoring function, and should be modeled as an equilibrium rather than a maximizing assignment \cite{WZB2011, GJ2017}.   

Previous work has considered outcomes as {\em pure strategy} Nash equilibria (PSNE) of some fixed underlying {\em graphical game} \cite{GJ2017, GJ2016, HO2015, IO2014}, where the payoff of a player depends on its own strategy and the (aggregate) strategy of its neighbors in a graph. The advantage of pure strategies is that the observations can be directly related to actions taken in the game. The parameters of the game such as interactions are thus readily adjustable on the basis of observed action profiles. A drawback of this family of models is that PSNE do not exist in a wide class of games \cite{TV2007}, and they require players to have sufficient (complete) information about the actions of other players. Moreover, the setup entirely side-steps the issue of game dynamics, i.e., how the equilibrium is arrived at in the game, reducing the ability to use equilibrium as a predicted outcome in a context dependent manner. The key estimation procedures in most PSNE based approaches \cite{GJ2017, GJ2016} are also combinatorially hard. From the point of view of applicability, the models are also tailored to a fixed set of players, and thus do not enable inference about the behavior of new players.   

We address these issues by extending the scope of structured prediction to games. At a high level, our model takes the available context such as a bill or a resolution to be voted on -- the input -- and maps it to a predicted outcome which is an action profile. We model the impact of context by parametrically mapping it to private information for each player (player types) which is subsequently incorporated into player utilities and their initial strategy profiles. We adopt mixed strategy Nash equilibria (MSNE) that exist in any game, unlike PSNE. Since player types are hidden, we call our parametrized games Latent Aggregative Games (LAGs). Our games can be viewed as a conditional version of directed graphical games \cite{KLS2001, KM2002}, restricted to a rich subclass called aggregative games that subsumes Cournot oligopoly, mean field, public goods, and population games \cite{GJ2017}. Aggregative games shield each player from specific information about any neighbor since players respond to aggregate (weighted sum) of their neighbors' strategies.   

A key novelty of our approach is to explicitly incorporate game dynamics that specifies how the predicted equilibrium is derived from the context. 
In our approach, we follow a $k$-step best response dynamics, seeded with predicted initial strategies, to arrive at a (near) mixed strategy equilibrium. An observed outcome, in response to the context, is then viewed as a sample from this predicted mixed strategy equilibrium. As we operate on continuous strategies and our updates are differentiable functions, we can use standard back-propagation to evaluate gradients through the $k$-step strategy updates, and thus learn parameters efficiently, unlike e.g. \cite{GJ2017, GJ2016}. We also generalize strategic prediction to transferable games where the sets of players may change from one game to another. To this end, we incorporate players into the game in terms of their feature representations, and learn a mechanism for mapping these features and the context into payoffs, strategic interactions as well as initial strategies. These games permit us to predict the behavior of new players in new contexts.
 
Since game dynamics plays a critical role in strategic prediction, we introduce and  provide a deeper analysis of more general dynamics and types of aggregation strategies. Our work pins down exact conditions under which strategies converge to different types of equilibrium. Convergence
to Nash equilibria has been known largely for a restricted class of games, e.g., two player zero-sum and potential games \cite{HW2003}, and multiplayer games are known to be considerably more difficult \cite{BRMFTG2018, FCAWAM2018, DISZ2018, HRUNH2017, MNG2017, NK2017, MPP2018, HM2013}. 
Our analysis makes use of tools from control theory, dynamical systems, and stochastic approximation \cite{B2008, KY2003, BHS2005, BHS2006, BPP2013, CL2003}, and thus contributes to this line of work as well. 

Finally, we provide identifiability guarantees for LAGs. For the analysis, we adopt a simpler {\em one-shot} setting, where the observed outcome is sampled from player strategies after one round of communication instead of following $k$-steps to a near mixed strategy equilibrium. We characterize conditions under which one-shot LAGs become identifiable, i.e., we can recover the neighbors of any player with the correct sign of their interaction (positive or negative). Such recovery is infeasible under PSNE since multiple game structures may pertain to the same set of PSNE \cite{HO2015, GH2017}. 

%% file: setting.tex
\section{Basic strategic prediction model} \label{Setting}
We first introduce our basic strategic prediction model. To this end, we need to define several components of the model. These include (a) the graphical layout of the game, and how players influence each other; (b) the player types (private information) and how these are derived from the context; (c) initial strategies for the players before witnessing the play of others; (d) individual utilities for the players; (e) and the game dynamics, i.e., how players respond to others. Later, in the transferable setting, we will no longer individuate players through their identities but instead adopt feature vectors for players.

Let $G = (E, V)$ be a connected digraph such that vertex $i$ identifies player $i \in [n] \triangleq \{1, 2, \ldots, n\}$, where $n = |V|$. Let $A$ be the (common) finite discrete set of actions for all the players, and let $T_i \subseteq \mathbb{R}^{|A|}$ be the latent type set for player $i$, defined shortly. Each player $i$ plays a randomized (mixed) strategy which is a distribution over actions: $\sigma_i \in \triangle(A)$ such that 
$\sum_{a_i \in A} \sigma_i (a_i) = 1$ and  $\sigma_i(a_i) \geq 0$ for $a_i \in A$. 
We will denote a joint strategy profile of all the players by $(\sigma_i, \sigma_{-i})$ to emphasize the distinction between player $i$ and all others. 
We model the influence of players on others through weighted aggregation of neighbor strategies. The weights $w_{ij} \in \mathbb{R}$ denote the strength of influence of player $j$ on player $i$. We will call players $j \in [n]\setminus\{i\}$ that have $w_{ij} \neq 0$ the neighbors of players $i$. We define a weight matrix $W \in \mathbb{R}^{n \times n}$ such that $W(i, i) = 0$ and $W(i, j) = w_{ij}$.  Player $i$ communicates with other players only through aggregator $\mathcal{A}_i$ that maps the strategies of other players, i.e., $\sigma_{-i}$ to the weighted sum $\sum_{j \neq i} w_{ij} \sigma_j$, the effective influence of others. The context influences the game through private types of players. This mapping could be defined in multiple ways. For simplicity, we initially parameterize the private type of each player $i \in [n]$ by a linear transformation or matrix $\theta_{i} : \mathcal{X} \to  T_i$ that maps context $x \in \mathcal{X} \subseteq \mathbb{R}^d$ to $z_i(x) = \theta_{i} x$.  We will keep the dependence on context implicit, for simplicity, and abbreviate $z_i(x)$ as $z_i$ from here on when the context is clear. We model the utility or payoff of player $i \in [n]$, of type $z_i \in T_i$, under strategy profile $(\sigma_i, \sigma_{-i})$ as:
 \begin{equation} \label{Payoff} U_i(\sigma_i, \sigma_{-i}, z_i) =  \sigma_i^{\top} \left(\mathcal{A}_i(\sigma_{-i}) - z_i\right)  + \tau \mathbb{H}(\sigma_i)~, \end{equation}
 where $\mathbb{H}(\sigma_i)$ is the entropy associated with $\sigma_i$ and $\tau \geq 0$. The entropy encourages {\em completely mixed} strategy choices, in the interior of simplex $\Delta(A)$. Our payoff functions generalize {\em linear influence games} that describe several decision scenarios such as whether to vaccinate against a disease, install antivirus software, or get home insurance \cite{IO2014, HO2015, GH2017}. We allow multi-way actions and private types, thus capturing a wider range of strategic behaviors. 
The payoffs may be interpreted as the expected reward received by players in a repeated game. 
 %
 We can naturally define the {\em best response} of player $i$ when it observes the aggregate input $\mathcal{A}_i(\sigma_{-i})$ to be
 \begin{equation} \label{beta} \beta_{i}^{\tau}(\mathcal{A}_i(\sigma_{-i}), z_i) \in \arg\!\!\!\!\max_{\sigma_i \in \triangle(A_i)} U_i(\sigma_i, \sigma_{-i}, z_i)~. \end{equation}
 We say that $(\sigma_i^*, \sigma_{-i}^*, z_{i}, z_{-i})$ 
forms an MSNE or simply a Nash equilibrium (NE) of LAG iff 
\begin{equation} \label{defMSNE} U_i(\sigma_i^*, \sigma_{-i}^*, z_i) ~~\geq~~  U_i(\sigma_i, \sigma_{-i}^*, z_i)~ \qquad \forall i \in [n], \sigma_i \in \triangle(A)~. \end{equation}
Every finite game has at least one MSNE \cite{N1951}. We say that MSNE is {\em strict} (SNE) when \eqref{defMSNE} is strict for all $i \in [n], \sigma_i \in \Delta(A)\setminus\{\sigma_i^*\}$,  {\em completely mixed} (CMNE) when $\sigma_i^*(a_i) > 0$ for all $i \in [n], a_i \in A$, and {\em pure} (PSNE) when for all $i \in [n]$ there exists an $a_i \in A$ such that $\sigma_i^*(a_i) = 1$.  
It remains to specify how an equilibrium is reached, i.e., the game dynamics. To begin with, players observe context $x$, and evaluate types $z_i$. Our setting dispenses with the restrictive assumption made by Bayesian games \cite{H1967, K2004, JB2010} that the conditional distribution $P(z_{-i}|z_i)$ is known to player $i$. In our case, the types give rise to initial strategies $\sigma_i^0 = \psi(z_i)$, where $\psi : T_i \to \Delta(A)$ (e.g. softmax). The best response dynamics from this point on depends on the details of the aggregator and whether the dynamics is defined over strategies or actions directly. We study several alternative game dynamics with different aggregators in Section \ref{SecDynamics}. In our empirical analysis, we adopt a simpler $k$-step corrective dynamics as described below. Once a (near) equilibrium is reached, a sample action profile $y \in \mathcal{Y} \subseteq A^n$ is drawn from player strategies.  

%% file: learning.tex
\textbf{Parameter estimation.} \label{ParEst}
We learn our games from data as structured prediction methods. Specifically, given a dataset $D = \{(x^{(m)}, y^{(m)}) \in \mathcal{X} \times \mathcal{Y}, m \in [M]\}$ linking contexts to sampled action profiles, our objective is to estimate the type parameters $\theta_i$ and the influences of neighbors $w_i \triangleq (w_{ij})_{j \neq i}, i \in [n]$. 
Each pair $(x^{(m)}, y^{(m)})$ is treated as follows. A linear transformation $\hat{\theta} = 
(\hat{\theta}_1, \ldots, \hat{\theta}_n)$ maps the context $x^{(m)}$ to the types $\hat{z}(x^{(m)}) \triangleq (\hat{z}_1^{(m)}, \ldots, \hat{z}_n^{(m)})$ that result in initial strategies $\hat{\sigma}_i^0(x^{(m)}) = \zeta(\hat{z}_i^{(m)})$ of the players $i \in [n]$, where $\zeta$ is the softmax nonlinearity. The aggregators $\hat{\mathcal{A}}_i$ evaluate weighted sums, and are parametrized by weights $\hat{w}_i =  (\hat{w}_{ij})_{j \neq i}$. A sequence of $k$ update steps 
\begin{equation} \label{gChoice} \hat{\sigma}_i^{t + 1}(x^{(m)}) = \zeta(\nu(\hat{\sigma}_i^{t}) + \alpha (\hat{\mathcal{A}}_i(\hat{\sigma}_{-i}^t) - \hat{z}_i^{(m)})), ~~~t = 0, 1, \ldots, k-1, \end{equation}
is then followed: $k$ and $\alpha$ are hyperparameters, 
and $\nu$ defines the type of update. Several choices of $\nu$ are possible; e.g., $\nu(\hat{\sigma}_i^{t}) = 0$ pertains to best response $\beta_i^{1/\alpha}(\hat{\mathcal{A}}_i(\hat{\sigma}_{-i}^t), \hat{z}_i^{(m)})$ defined in \eqref{beta}, and the identity mapping $\nu(\hat{\sigma}_i^{t}) = \hat{\sigma}_i^{t}$ defines a gradient step.
Our estimation criterion for the game is to minimize the expected cross-entropy loss $\mathbb{E}[\ell(\hat{\sigma}^k(x^{(m)}), y^{(m)}]$ between the predicted mixed strategies and the observed profiles, where the expectation is with respect to the empirical distribution over pairs $(x^{(m)}, y^{(m)})$, $\ell$ is the cross entropy loss, and  $\hat{\sigma}^k(x^{(m)}) \triangleq (\hat{\sigma}_i^k(x^{(m)}))_{i \in [n]}$. We use standard backpropagation to evaluate gradients through the $k$-step strategy updates efficiently. 

%% file: transfer.tex
\section{Transferable strategic prediction} \label{SecTrans} 

Here we generalize LAGs to permit different players from one game to another. Unlike in section \ref{Setting}, we can no longer assume a fixed interaction structure across games. Instead, the neighbor influences are determined by context and player {\em feature vectors}. The model enables us to predict the behavior of new players in new contexts. 
Specifically, we construct a feature vector $b_i \in \mathcal{B}$ for each player $i \in [n]$. Such information is often publicly available, e.g., education and gender of judges; human development indicators of countries, etc.  Each game is played with a different subset of players $\mathcal{I} \subseteq [n]$, and is unrolled as follows. A context $x \in \mathcal{X}$ is mapped to latent (more general) player types using a parametric function $f_{z} : \mathcal{X} \times \mathcal{B}  \to \mathcal{Z}$, taking each pair $(x, b_v), v \in \mathcal{I}$ as input, and mapping it to $z_{x, v} \in \mathcal{Z}$. The latent types define initial strategies as before
$\sigma_{x, v}^0 = \phi(\Gamma z_{x, v})$, where $\Gamma$ is a transformation matrix that yields a vector in $\mathbb{R}^{|A|}$, and $\phi$ (e.g., softmax) maps the result to a distribution in the simplex $\Delta(A)$. Unlike before, the (asymmetric) influences between players are now calculated parametrically from the types: $w_{x, v, v'} = f_w(z_{x, v}, z_{x, v'})$ using a parametric mapping $f_w : \mathcal{Z} \times \mathcal{Z} \to \mathbb{R}$. Each player $v$ still responds to other players $v' \in \mathcal{I}\setminus\{v\}$ through its aggregator 
$$\mathcal{A}_{x, v, \mathcal{I}}(\sigma_{x, -v}) \triangleq \sum_{v' \in \mathcal{I}: v' \neq v} w_{x, v, v'} \sigma_{x, v'}~.$$
We can extend the definition of the payoffs slightly to incorporate the more general latent types: 
\begin{equation*} \label{TransferPayoff} \hspace*{-0.2cm} U_{v, \mathcal{I}}(\sigma_{x, v}, \sigma_{x, -v}, z_{x, v})=\sigma_v^{\top} \left(\mathcal{A}_{x, v, \mathcal{I}}(\sigma_{x, -v})) - \Gamma z_{x, v}\right). 
\end{equation*}
where $\Gamma$ is an additional parameter matrix to be learned.
The game dynamics dictates the course of play in the same fashion as the basic strategic setting. We learn the model, now parameterized by $f_z$, $f_w$, and $\Gamma$, by minimizing the loss between predicted $k$-step strategies and observed action profiles. 

%% file: fp.tex
\section{General game dynamics and convergence} \label{SecDynamics}
We now provide an in-depth look at the game dynamics along with associated convergence guarantees. 
The aggregator in the game acts as a privacy preserving component, hiding specific neighbor actions or strategies, only offering aggregate statistics. We design dynamics under two different kinds of feedback from the aggregator. In an  {\em active aggregator} (AA) setting, the players get a prediction about the anticipated aggregate of their neighbors. In contrast, a {\em passive aggregator} (PA) only provides the aggregate of {\em empirical frequencies} used by the neighbors, and changes in the aggregate are estimated by the player. Intuitively, AA reveals more information about the neighbors' strategy evolution. We devise two new protocols as derivative action adaptations of smooth fictitious play (FP) and gradient play (GP) \cite{B1951, R1951, FK1993, S1964, SA2005} for aggregative games. The protocols differ by how players respond to the (predicted) aggregate: one can play the best response or adapt the strategy via a gradient update. 
Formally, player $i$ samples an action $a_i^k \sim \sigma_i^{k}$ at time $k > 0$
based on
\begin{eqnarray} 
q_i^k & = & q_i^{k-1} + (e_{a_i^{k-1}} - q_i^{k-1})/k \in \Delta(A)~; \qquad  
\sigma_i^{k}  ~=~  g_i(\mathcal{A}_i(h_{-i}(q_{-i}^k)), z_i)~, \;\;\;  \label{general} \end{eqnarray}
where $q_i^k$ is the empirical frequency of actions played by $i$ till time $k$, and  $g_i : \mathbb{R}^{|A|} \times  \mathbb{R}^{|A|} \to \Delta(A)$ and $h_i : \mathbb{R}^{|A|} \to \mathbb{R}^{|A|}$ are appropriately defined Lipschitz mappings possibly involving small input noise. We let $q_{-i}^{k-1} \triangleq \{q_j^{k-1} | j \neq i\}$, and $h_{-i}(q_{-i}^{k-1}) \triangleq \{h_j(q_j^{k-1}) | j \neq i\}$.  We also define the base case $q_i^0 = \sigma_i^{0} = \phi(z_i)$. Note that player $i$ communicates only with $\mathcal{A}_i$.  
We define a {\em passive} aggregator (PA) by letting $h_i$ be the identity mapping, i.e. $h_i(q_i) = q_i$. Alternatively, when $h_i(q_i) = q_i + \gamma \nabla \tilde{q}_i$ for some $\gamma > 0$ and a difference approximation $\nabla \tilde{q}_i$ of a temporal derivative $\nabla q_i$, we obtain an {\em active} aggregator (AA). 
Intuitively, AA views each $q_j$ as discretization of a continuous signal $q_j(t)$ so that when $\nabla \tilde{q}_j(t) \approx \nabla q_j(t)$, for neighbors $j$ of $i$, we have 
\begin{eqnarray*} h_j(q_j(t)) & \approx & q_j(t) +  \gamma \nabla q_j(t) \approx q_j(t + \gamma) \implies   \mathcal{A}_i(h_{-i}(q_{-i}(t))) \approx \mathcal{A}_i(q_{-i}(t + \gamma)), \end{eqnarray*}
and therefore $\mathcal{A}_i$ offers a predicted aggregate to player $i$.
We consider two forms of best response dynamics encoded in $g_i$, LAG-FP and LAG-GP, based on derivative FP and derivative GP, respectively.  In LAG-FP we set $\tau > 0$ in the utility functions. This lets us have a unique best response: 
\begin{eqnarray*} \label{LAG-FP} 
 (\textbf{LAG-FP})\qquad  \text{AA yields }  \overline{u}^k:  ~  g_i(\overline{u}^k, z_i) & = &   \beta_{i}^{\tau}(\overline{u}^k, z_i), \\ 
\text{PA yields }  u^k: ~  g_i(u^k, z_i) & = &  \beta_{i}^{\tau}(u^k + \gamma \nabla \hat{u}^k, z_i), 
\end{eqnarray*}
where the AA case differs from PA in terms of where the difference approximation happens. In AA, it happens prior to aggregation thus $g_i$ is defined directly in terms of the output of AA or $\bar{u}^k$ which absorbs any temporal approximation error. In PA, the player constructs a temporal prediction of the aggregate, and the approximation is $(\nabla u^k - \nabla \hat{u}^k)$. In LAG-GP, we set $\tau = 0$, and player $i$ takes a gradient step to maximize the anticipated payoff followed by a Euclidean projection to get a unique mapping $g_i$ (since any such projection on a closed convex set is unique):  
\begin{eqnarray*} \label{LAG-GP} 
\hspace*{2cm} (\textbf{LAG-GP}) \qquad  \text{AA yields }  \overline{u}^k: ~~ g_i(\overline{u}^k, z_i) & = &   \Pi_{\Delta} (q_i + \overline{u}^k - z_i), \qquad \qquad \qquad \qquad \qquad \\
 \text{PA yields } u^k: ~~ g_i(u^k, z_i) & = &    \Pi_{\Delta} (q_i^k + u^k + \gamma \nabla \hat{u}^k - z_i), \qquad \qquad \qquad ~
\end{eqnarray*}
where $\Pi_{\Delta} (q)  \triangleq   \arg\!\min_{\tilde{q} \in \triangle(A)} ||\tilde{q} - q||_2~.$
Thus under both protocols, players take actions stochastically according to $\sigma_i^{k}$ and the best response mapping $g_i$ is unique for each $k$. Assuming that the error sequence in updating $\{\sigma_i^{k}\}$ is a martingale, 
our updates satisfy the conditions in (section 2.1, \cite{B2008}) and we can analyze the stochastic evolution of each LAG as a noisy discretization of a limiting ordinary differential equation (ODE). 
We investigate the conditions under which the fixed points of the this ODE are {\em locally asymptotically stable}, and as a consequence, our discrete updates would converge to a Nash equilibrium with positive probability \cite{SA2005}.  An equilibrium point $s$ is locally asymptotically stable if every ODE trajectory that starts at a point in a small neighborhood of $s$ remains forever in that neighborhood and eventually converges to $s$.
\begin{figure*}
\begin{eqnarray}
(\textbf{LAG-FP/AA}) ~~~ ~ \dot{q}_i & = & \beta_{i}^{\tau}(\mathcal{A}_i(q_{-i} + \gamma \dot r_{-i}), z_i)- q_i, \qquad \qquad \hspace*{0.1cm} \dot{r}_i  = \lambda(q_i - r_i)    \label {DLAFP-True} \\
(\textbf{LAG-FP/ PA}) ~~~ ~\dot{q}_i & = & \beta_{i}^{\tau}(\mathcal{A}_i(q_{-i}) + \gamma \dot r_{i}, z_i)- q_i, \qquad \qquad ~~~ \dot{r}_i  =  \lambda(\mathcal{A}_i(q_{-i}) - r_i) \label {DLAFP-True-Player}\\
(\textbf{LAG-GP/AA}) ~~~  ~ \dot{q}_i & = & \Pi_{\Delta} [q_i + \mathcal{A}_i(q_{-i} + \gamma \dot r_{-i}) - z_i]- q_i, \hspace*{0.2cm}~~~\dot{r}_i  =  \lambda(q_i - r_i) \label{DLAGP-ODE}\\
(\textbf{LAG-GP/PA}) ~~~ ~ \dot{q}_i & = & \Pi_{\Delta} [q_i + \mathcal{A}_i(q_{-i}) + \gamma \dot r_{i} - z_i]- q_i, \hspace*{0.25cm}~~~~~  \dot{r}_i  =  \lambda(\mathcal{A}_i(q_{-i}) - r_i) \label{DLAGP-ODE_Player}
\end{eqnarray}
\vskip -0.19in
\end{figure*}
Our updates in \eqref{general} lead to the implicit ODEs \eqref{DLAFP-True}-\eqref{DLAGP-ODE_Player} for LAG-FP and LAG-GP under AA and PA settings, 
where $\lambda > 0$, $\dot r_i$ is an estimate for $\dot q_i$, and $\dot r_{-i} \triangleq \{\dot r_j  |  j  \neq i, w_{ij} \neq 0\}$. We call a matrix {\em stable} if all its eigenvalues have strictly negative real parts. Let $I$ denote the identity matrix. We now specify conditions under which dynamics converge to CMNE or SNE in the PA setting. We defer the convergence to NE in the PA setting, and all convergence results in the AA setting to the supplementary material. 

\begin{theorem} {\bf (LAG-GP/PA convergence to CMNE)} \label{Theorem5}
Let the weight matrix $W$ be stochastic. Let $(q_1^*, \ldots, q_n^*, z_1, \ldots, z_n)$ be a completely mixed NE under the dynamics in \eqref{DLAGP-ODE_Player}.  The linearization of \eqref{DLAGP-ODE_Player} with $\gamma > 0$ is locally asymptotically stable for $\lambda > 0$   if and only if the following matrix is stable
$$\begin{bmatrix}
     (1 + \gamma \lambda) W &  -\gamma \lambda W \\
    \lambda W & -\lambda I
\end{bmatrix}.
$$
\end{theorem}

\begin{theorem} {\bf (LAG-GP/PA convergence to SNE)} \label{Theorem6}
Let the weight matrix $W$ be doubly stochastic. Let $(q_1^*, \ldots, q_n^*, z_1, \ldots, z_n)$ be a strict NE under the dynamics in \eqref{DLAGP-ODE_Player}.  The equilibrium point $(q_i = q_i^*, r_i = A_i(q_{-i}^*))_{i \in [n]}$ is locally asymptotically stable for sufficiently small $\gamma \lambda$, where $\gamma, \lambda > 0$. 
\end{theorem}
Our results have important implications. \cite{GJ2017, GJ2016} enforced margin constraints  on payoffs in their PSNE setups. They did not establish SNE, which guarantees a {\em strictly} worse payoff to any player that unilaterally deviates from equilibrium. Theorem \ref{Theorem6} specifies the conditions for convergence to SNE, and dictates when margin constraints should be imposed. The classic fictitious play (FP) fails to converge in simple games, e.g.  \cite{S1964}, that have a unique CMNE.  Theorem  \ref{Theorem5} specifies conditions that circumvent such negative results. We prove our results via carefully crafted {\em Lyapunov} and {\em Hurwitz} stability analyses.
We provide all the detailed insights and proofs in the supplementary material. 

%% file: experiments2.tex
\section{Identifiability of the games} \label{SecIdenti}
We now characterize the conditions under which LAGs become identifiable in terms of strategic interactions. Our recovery procedure is a novel adaptation of the primal-dual witness method \cite{W2009} to games. Specifically, we estimate from data $D$ the {\em support} $S_i$ or the set of neighbors of $i$ defined wrt to the unknown influences $w_{ij}^* \neq 0$. We also recover the correct sign of these influences. We focus on the {\em one-shot} setting (i.e., $k = 1$) with the gradient update so that the observed outcome is sampled from player strategies after one round of communication. We also use binary actions to simplify the exposition. We denote by $\phi_j^{(m)}$ the probability assigned by the initial strategy $\sigma_j^{0}(x^{(m)})$ to action $1$ for player $j$ on example $m$. Let $\ell_i(w_i; D)$ be the average cross-entropy loss between one-step strategies under candidate weights $w_i \triangleq (w_{ij})_{j \neq i}$ and the observed actions for player $i$, and  $\lambda> 0$ be a regularization parameter. We solve the following problem for each player $i \in [n]$: 
\begin{equation} \label{logistic1}
\arg\!\!\min_{w_i \in \mathbb{R}^{n-1}}  \ell_i(w_i; D) ~+~ \lambda ||w_i||_1~, 
\end{equation}
Let $H_{i}^M$ be the sample Hessian $\nabla^2\ell_i(w_i; D)$ under $w_i$, and $H_{i}^{*M}$ pertain to true weights $w_i^*$. Let $\Lambda_{min}(\cdot)$ and  $\Lambda_{max}(\cdot)$ denote the minimum and the maximum eigenvalues. The following assumptions serve as our analogues of the conditions for support recovery in Lasso \cite{W2009} and Ising models \cite{RWL2010}:
\begin{eqnarray} \label{min_max_eigena}
\Lambda_{\min}\left(H_{i, SS}^{*M}\right) & \geq &  \alpha^2C_{min},~ ~~~ |||H^{*M}_{i, S^cS} (H^{*M}_{i, SS})^{-1}|||_\infty ~\leq~ 1 - \gamma \\
\label{incoherencea}
\qquad \Lambda_{\max}\left(\dfrac{1}{M} \sum_{m=1}^M \Phi_{-i}^{(m)} \Phi_{-i}^{(m)^\top}\right) & \leq & C_{max},~~~~ \text{where}~~~~ \Phi_{-i}^{(m)} ~~\triangleq~~ (\phi_j^{(m)})_{j \neq i}
~,
\end{eqnarray}
where $C_{\min} > 0$, $C_{\max} < \infty$, $\gamma \in (0, 1]$;   $|||A|||_\infty$ is the maximum over $L_1$-norm of rows in $A$; $H_{i, SS}^{*M}$ is the submatrix obtained by restricting $H_{i}^{*M}$ to rows and columns corresponding to neighbors, i.e., players in $S_i$, and $H_{i, SS^c}^{*M}$ is restricted to rows pertaining to $S_i$ and columns to $S^c_i$ (non-neighbors). Let the number of neighbors for any player be at most $d \leq n-1$. We have the following result. 
\begin{theorem}
Let $M > \dfrac{80^2C_{\max}^2}{C_{\min}^4} \left(\dfrac{2 - \gamma}{\gamma}\right)^4 d^2 \log(n)$, and $\lambda \geq \dfrac{8 \alpha (2 - \gamma)}{\gamma} \sqrt{\dfrac{\log(n)}{M}}$. Suppose the data satisfies assumptions \eqref{min_max_eigena}, and \eqref{incoherencea}. The following results hold with high probability for each $i \in [n]$: (a) the corresponding optimization problem \eqref{logistic1} has a unique solution, i.e., a unique set of neighbors for $i$, and (b)  the set of predicted neighbors of $i$ is a subset of the true neighbors. Additionally, the predicted set contains all true neighbors $j$ for which $|w_{ij}^*| \geq 10\sqrt{d} \lambda/(\alpha^2C_{\min})$.
\end{theorem}
\section{Experiments} \label{SecExp}
We now describe the results of our experiments that provide insights into some important aspects of our games.  We first show that LAG qualitatively recovers the known strategic behavior of the Justices in the longest serving US Supreme Court. We also provide quantitative evidence that LAG outperforms the prior methods on two different measures.  
We then demonstrate that the structure estimated by LAG on the UN General Assembly data \cite{VSB2009} is meaningful, and helps unravel the subtle behavior of member countries.  Finally, we present evidence to underscore that LAGs can be effectively transferred to predict strategies in new settings with different sets of players. 

We found that LAGs performed well over a  wide range of hyperparameters. We implemented the models in sections \ref{ParEst} and \ref{SecTrans} with $L_1$ regularization for structure estimation as described in section \ref{SecIdenti}. Our models performed well for a wide range of $\alpha$, $\lambda$, and $k$. We report the results with $k = 5$, $\alpha = 0.1$, and $\lambda = 0.1$ for all our experiments, except the transferable setting where we set $\alpha = 0.01$ and $\lambda = 0$.\footnote{We did not impose  $L_1$ penalty in the transferable setting since the interactions are learned for new individuals, i.e., they are not specific to any fixed set of players, unlike the basic setting.} We set $\nu$ to be the identity function in \eqref{gChoice}. We trained our models in batches of size 200, with default settings of the RMSprop optimizer in PyTorch.  To account for effect of randomness in neural training toward structure estimation, we averaged the parameters of each model across 5 independent runs. We ordered the influence of neighbors based on the corresponding estimated average weights from the most positive to the most negative.  
\subsection{US Supreme Court Data}
We included all the cases from the {\em Rehnquist Court}, during the period 1994-2005 that had votes documented for all the 9 Justices (i.e. our players).  Justices Rehnquist (R), Scalia (Sc),  and Thomas (T) represented the conservative side; Justices Stevens (St), Souter (So), Ginsburg (G), and Breyer (B) formed the liberal bloc; and Justices Kennedy (K) and O'Connor (O), often called {\em swing votes}, followed a moderate ideology. Our contexts comprise of 32 binary attributes that characterize the specifics of the appeal, e.g., the disposition of lower court. Our observation outcome $y^{(m)}$ for each context $x^{(m)}$ pertains to the corresponding votes of the Justices. The votes belong to one of the three categories: yes, no, or complex. 
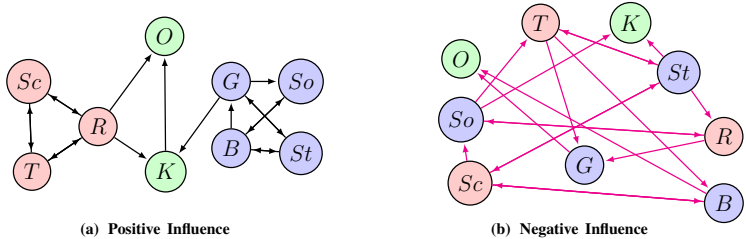
\begin{figure*}
\centering
\scalebox{.6}{\input{RehnquistL1.tex}}
   \caption{{\bf Supreme Court structure recovery with LAGs}:  (a) and (b) show, respectively, justices with the most positive and the most negative influence, quantified by $\hat{w}_{ij}$, on each Justice $i$. The estimated connections are consistent with the known jurisprudence of the Court. In particular,  (a) shows coherence between the conservatives (red),  that between the liberals (blue), and the separation of these ideologies from the moderates (green). Likewise, (b) shows all the negative connections are between the blocs. The moderates K and O do not have  outgoing connections to the liberals and the conservatives. This emphasizes, in particular, that the moderates espouse a centrist viewpoint and do not exercise strong influence on others, positive or otherwise. Note that determining influences based on heuristics like ordering by pairwise vote agreements does not work; e.g, that would imply K had a strong positive influence on R, since R agreed with K more than with anyone else.} \label{fig:SupremeCourt}
  \end{figure*}   
\begin{figure*}[t]
\resizebox{12cm}{!}{\input{RehnquistCompareL1.tex}}
  \caption{{\bf Quantitative comparison with prior methods}:  (a), (b), and (c) show the structures estimated by LAG, local AG \cite{GJ2017}, and tree structured potential game \cite{GJ2016} on US Supreme Court data.  Since only the conservatives (red nodes) and liberals (blue nodes) are known to influence the moderates (green nodes) and not vice-versa,  we define a {\em correct edge} to be one that either (1) connects two nodes of  same color in any direction, or (2) goes outward from a red or blue node into a green node. One way to quantify the quality of recovery is to compute the fraction of correct edges. All 18 edges recovered by LAG are correct, and therefore its recovery score is maximum possible, i.e., 1. In contrast, Local AG \cite{GJ2017} gets edges \{(O, R), (O, T), (O, Sc), (O, B), (K, Sc), (K, T)\} wrong for a score of only 16/22, i.e. 0.73. Finally, \cite{GJ2016} yields undirected edges. Treating each edge as bidirectional, we note that \cite{GJ2016} estimates 16 edges out of which it makes mistakes on \{(O, So), (St, R), (R, St), (St, Sc), (Sc, St)\}, and  thus registers a score of  11/16, i.e., 0.69. Yet other way to evaluate is a {\em cut}, i.e., minimum number of edges to be removed to decompose the structure into its three components (reds, blues, and greens). A low value of cut quantifies high coherence within each component, and thus pertains to a good structure. The cut size for LAG (3) is much lower than other methods (6 each).} \label{fig:RehnquistCourt}  
  \end{figure*}
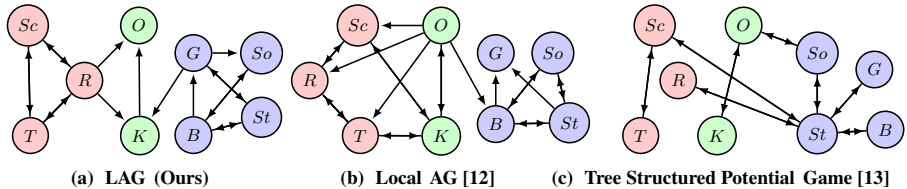 
Fig. \ref{fig:SupremeCourt} describes in detail how our method yields a structure that is qualitatively consistent with the known jurisprudence of the Rehnquist Court. Specifically, (a) the conservatives and the liberals form two separate coherent, strongly-connected blocs that are well-segregated from each other, and (b) they influence the moderates but not vice-versa. 

Fig. \ref{fig:RehnquistCourt} provides a detailed quantitative comparison that reveals how LAG
 compared favorably with both the prior PSNE based methods \cite{GJ2017, GJ2016} on two separate measures: (a) recovering edges consistent with the ideology of the Justices, and (b) coherence in terms of size of the {\em cut}, i.e., minimum number of edges to be removed in order to decompose the recovered structure into the constituent blocs. 
\subsection{United Nations General Assembly (UNGA) Data}
Our second dataset consists of the roll call votes of the member countries on the resolutions considered in the UN General Assembly. Each resolution is a textual description that provides a context while the votes of the countries on the resolution pertain to the observed outcome.  We compiled data on all resolutions in UNGA from 1992 onward to understand the interactions of member nations since the dissolution of the Soviet Union.  We considered 25 countries that have dominated the United States (USA) politics, and are generally known to belong to one of the two blocs: pro-USA, namely, Australia (AUS),  Canada (CAN), France (FRA), Germany (GER), Israel (ISR), Italy (ITA), Japan (JPN), South Korea (KOR), Norway (NOR), Ukraine (UKR); and others, namely,  Afghanistan (AFG), Belarus (BLR), China (CHN), Cuba (CUB), Iran (IRN), Iraq (IRQ), Mexico (MEX), Pakistan (PAK), Philippines (PHL), North Korea (PRK), Kazakhstan (KAZ), Russia (RUS), Syria (SYR), Venezuela (VEN) and Vietnam (VNM).  We used pretrained GLoVe  embeddings to represent each resolution as a 50-dimensional context vector $x^{(m)}$.
Each vote was interpreted to take one of the three values: 1 (yes), 2 (absent/abstain), or 3 (no), and we represented $y^{(m)}$ as a 26-dimensional vector.  

Fig. \ref{fig:UNGA} shows the structure estimated by our method, i.e., the weights $\hat{w}_i$ learned for each country $i$. The weights $\hat{w}_{ij}, j \neq i$ have a natural interpretation in terms of influence: the more positive $w_{ij}$ is, the more positive the influence of $j$ on $i$. A similar connotation holds for the negative weights. To aid visualization, we disentangle the positive and the negative connections, and depict only the most influential connections for either case. Fig. \ref{fig:UNGA} describes how our method estimated a meaningful structure,  and unraveled subtle influences beyond the prominent two-bloc structure. The learned type parameters $\hat{\theta}_i$ were also found to be similar for members in the same bloc (Fig. \ref{fig:UNGAType}).
\begin{figure*}
\centering
    \begin{subfigure}[b]{0.3\textwidth}    
            \includegraphics[trim=0cm 0.3cm 0.4cm 0.48cm,clip, width=\textwidth]{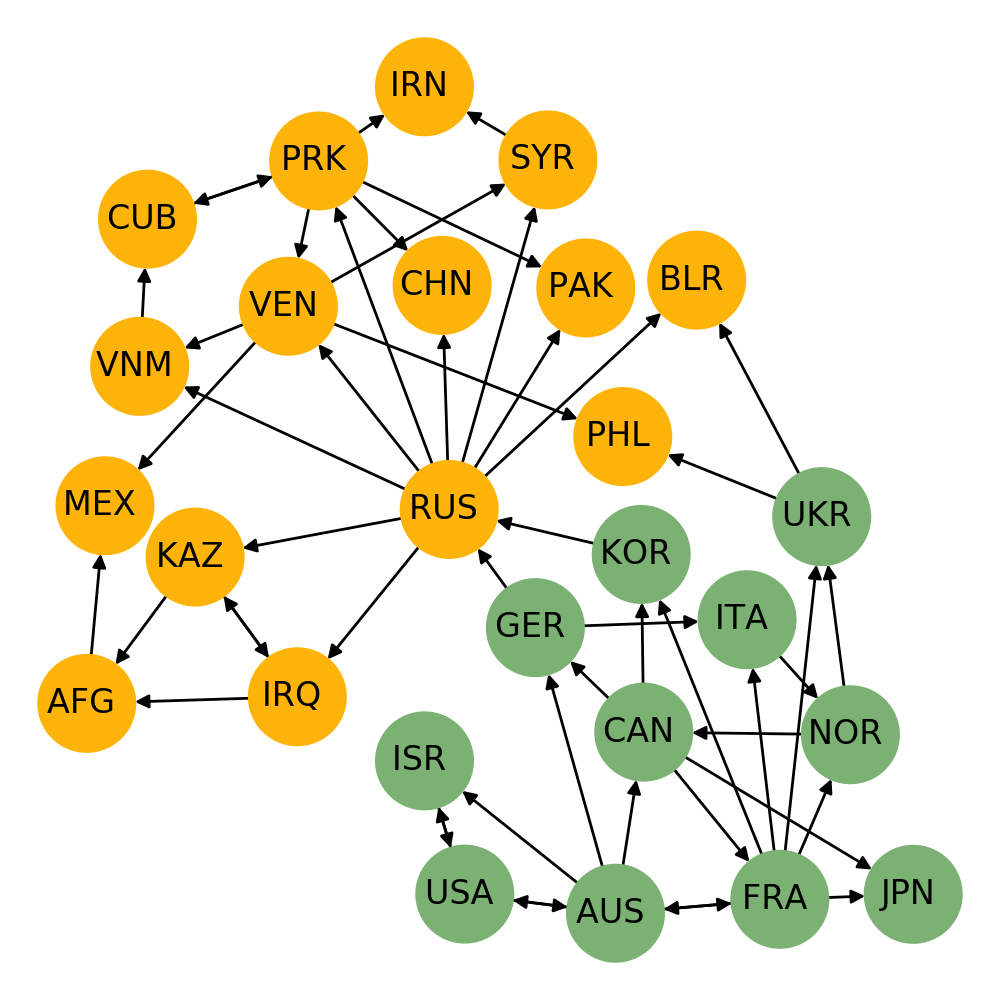} 
            \caption{{\bf Positive influence}}
            \label{fig:UNGAa}
    \end{subfigure}%
    \qquad  \qquad 
    \begin{subfigure}[b]{0.3\textwidth}
            \includegraphics[trim=0cm 0.3cm 0cm 0.48cm,clip,width=\textwidth]{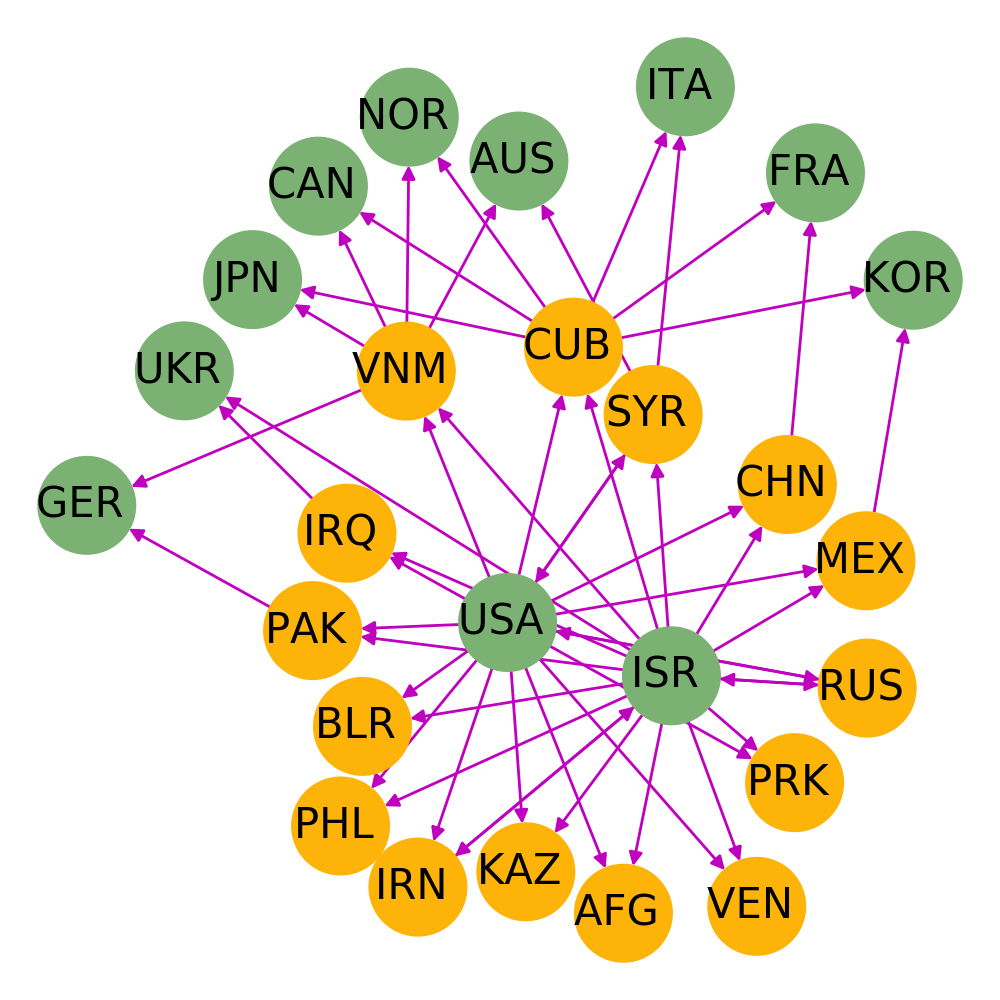}
            \caption{{\bf Negative influence}}
            \label{fig:UNGAb}
    \end{subfigure}
    \caption{{\bf UNGA structure recovery with LAGs}:  Incoming arrows show (a) 2 countries with the most positive influence (black edges), and (b) 2 countries with the most negative influence (magenta edges), quantified by $\hat{w}_{ij}$, on each country $i$. The estimated links are largely consistent with the expected alignments. In particular,  (a) shows the two blocs (in yellow and green) are well segregated from each other.  More interesting alignments are revealed, e.g., (1) strong affinity between NATO members on one side, and Syria, Iran, Venezuela etc. on the other, (2) link from Germany and Korea to Russia hinting at the trade influence despite their differences, and (3) geographical influence of Russia and Ukraine on Belarus.  Additionally, (b) reveals that a significant fraction of negative connections emanate from or end at Israel and USA on one side, and some yellow node on the other.}\label{fig:UNGA}
\end{figure*} 
\begin{figure*}
\centering
    \begin{subfigure}[b]{0.3\textwidth}            
            \includegraphics[trim=0cm 0.3cm 0cm 0.57cm,clip,  width=\textwidth]{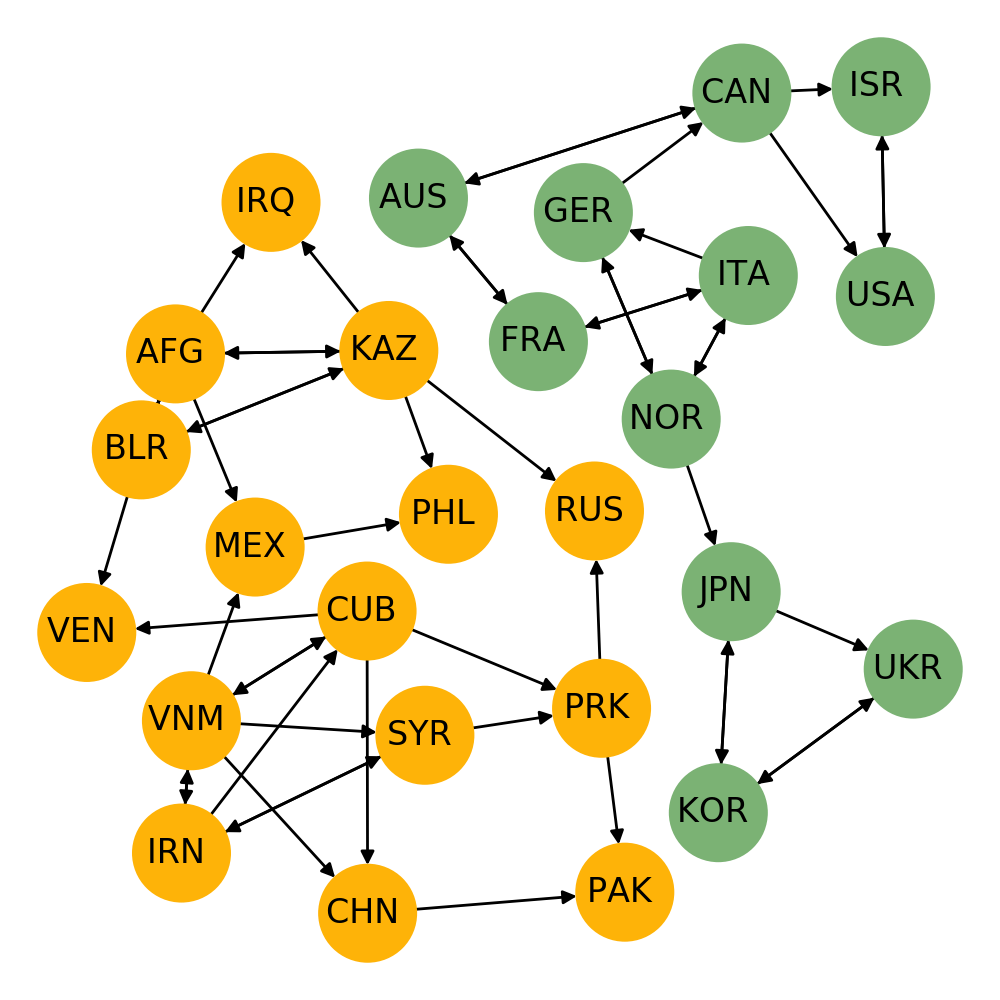}
            \caption{{\bf Type similarity}}
            \label{fig:TypePos}
    \end{subfigure}%
    \qquad \qquad
    \begin{subfigure}[b]{0.42\textwidth}
\includegraphics[trim=0cm 0cm 0cm 0.1cm,clip,width=\textwidth]{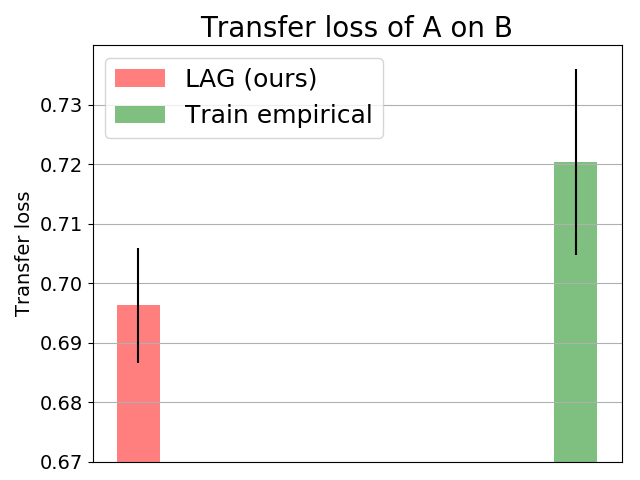}
\caption{{\bf Transfer performance} \label{fig:transfer}}
\end{subfigure}
 \caption{{\bf (Left)} Incoming arrows show 2 countries with the highest cosine similarity $\hat{\theta}_{i}^{\top} \hat{\theta}_j/(||\hat{\theta}_i||~||\hat{\theta}_j||)$ for each country $i$. 
 Type vectors were reasonably well aligned for members in the same bloc. {\bf (Right)} LAGs were more effective in predicting strategies for new players.  }\label{fig:UNGAType}
\end{figure*}

Our final set of experiments focused on  transferring LAGs. Besides UNGA data, we compiled a 73-dimensional feature vector for each country from its HDI Indicators \cite{UN2018}.
We then kept aside three-fourths of the data to set up games over small sets of players by randomly sampling 5 countries independently for each context. The left-out players from these contexts were then sampled independently to get another set of 5 countries per context. We call these two sets A and B. We formed a third set C of 5 countries per game from the untouched data (i.e. the one-fourth fraction). Thus, A and B were defined on the same contexts, that were disjoint from C. We averaged results over 10 such independent triplets (A, B, C) to mitigate sampling effects. We trained a model for A using the procedure in section \ref{SecTrans}, and computed the loss on B. Our baseline, {\em train empirical}, used the empirical distribution of actions for each player $i$ over the games it participated in A as its predicted strategy for the games in B and C. The loss of LAGs (0.852) turned out to be lower than the  baseline (0.865) on C. Moreover, as Fig. \ref{fig:UNGAType} shows, the loss of LAGs (0.696) was found to be significantly lower than the baseline (0.720) on B even though HDI does not fully reflect complex country characteristics. Thus, our results clearly underscore the benefits of using LAGs for strategic prediction.

%% file: RehnquistL1.tex
\begin{tikzpicture}[shorten >=1pt, auto, thick,
        node distance=0.6cm,
    con node/.style={circle,draw,fill=red!20,font=\sffamily\Large\bfseries},
    lib node/.style={circle,draw,fill=blue!20,font=\sffamily\Large\bfseries},
    mod node/.style={circle,draw,fill=green!20,font=\sffamily\Large\bfseries}]


\node[mod node] (OL) {$O$};
 \node[con node, left=of OL, yshift=-2cm]     (RL)   {$R$};
 \node[con node, left=of RL, yshift= 1cm]     (ScL)   {$Sc$};
 \node[con node, left=of RL, yshift= -1cm]    (TL)   {$T$};
 \node[mod node, right=of RL, yshift= -1cm]     (KL)   {$K$};
 \node[lib node, right=of OL, yshift = -1cm]     (GL)   {$G$};
 \node[lib node, right=of GL]     (SoL)   {$So$};
 \node[lib node, below=of GL]     (BL)   {$B$};
 \node[lib node, below=of SoL]     (StL)   {$St$};

\foreach \from/\to in {RL/ScL, TL/ScL, ScL/TL, RL/TL, ScL/RL, TL/RL, RL/KL, GL/KL, RL/OL, KL/OL, BL/GL, StL/GL, StL/BL, SoL/BL, GL/SoL, BL/SoL, BL/StL, GL/StL}
 \draw[-latex] (\from) edge (\to);
 
\draw (TL) -- node [below=1.5 cm, xshift=2cm] {{\bf (a)\,\, Positive\, Influence}} (RL);

\node[mod node, right=of SoL, xshift = 2cm, yshift = 0.5cm]  (OR) {$O$}; 
\node[con node, right=of OR, xshift=0.3cm, yshift = 0.8cm]  (TR) {$T$}; 
\node[mod node, right=of TR, xshift = 0.5cm]  (KR) {$K$};
\node[lib node, right=of OR, xshift = 3.3cm, yshift =-0.3cm]  (StR) {$St$}; 
\node[lib node, below=of OR, yshift = 0.2cm]  (SoR) {$So$}; 
\node[con node, below=of StR, yshift=0.1cm, xshift=1cm]  (RR) {$R$}; 
\node[lib node, below=of RR]  (BR) {$B$}; 
\node[con node, below=of SoR, xshift=0.2cm, yshift = 0.2cm]  (ScR) {$Sc$};
\node[lib node, right=of ScR, yshift=0.4cm, xshift=1cm]  (GR) {$G$};

\foreach \from/\to in {StR/RR, SoR/RR, TR/StR, ScR/StR, BR/OR, GR/OR, StR/ScR, BR/ScR, StR/KR, SoR/KR, RR/SoR, ScR/SoR, StR/TR, SoR/TR, TR/GR, RR/GR, TR/BR, ScR/BR}
 \draw[-latex] (\from) edge [magenta](\to); 

\draw[magenta] (ScR) -- node [yshift = -1.1cm,xshift=-2.5cm, color=black] {{\bf (b)\,\, Negative\, Influence}} (BR);
%
%
%
%
%
%
%
%
\end{tikzpicture}

%% file: RehnquistCompareL1.tex
\begin{tikzpicture}[shorten >=1pt, auto, thick,
        node distance=0.3cm,
    con node/.style={circle,draw,fill=red!20,font=\sffamily\small\bfseries},
    lib node/.style={circle,draw,fill=blue!20,font=\sffamily\small\bfseries},
    mod node/.style={circle,draw,fill=green!20,font=\sffamily\small\bfseries}]


\node[mod node] (OL) {$O$};
 \node[con node, left=of OL, yshift=-1cm]     (RL)   {$R$};
 \node[con node, left=of RL, yshift= 1cm]     (ScL)   {$Sc$};
 \node[con node, left=of RL, yshift= -1cm]    (TL)   {$T$};
 \node[mod node, right=of RL, yshift= -1cm]     (KL)   {$K$};
 \node[lib node, right=of OL, yshift = -0.5cm]     (GL)   {$G$};
 \node[lib node, right=of GL, xshift=0.2cm]     (SoL)   {$So$};
 \node[lib node, below=of GL, yshift=-0.5cm]     (BL)   {$B$};
 \node[lib node, below=of SoL, yshift=-0.1cm]     (StL)   {$St$};
 \foreach \from/\to in {RL/ScL, TL/ScL, ScL/TL, RL/TL, ScL/RL, TL/RL, RL/KL, GL/KL, RL/OL, KL/OL, BL/GL, StL/GL, StL/BL, SoL/BL, GL/SoL, BL/SoL, BL/StL, GL/StL}
 \draw[-latex] (\from) edge (\to);  
 
\draw (TL) -- node [below=1 cm, xshift=1.5cm] {{\bf (a)\,\, LAG\, (Ours)}} (RL);

\node[mod node, right=of OL, xshift = 4.5cm] (OM) {$O$};
 \node[con node, left=of OM, yshift=-1cm, xshift=-1.4cm]     (RM)   {$R$};
 \node[con node, left=of RM, yshift= 1cm, xshift=1.9cm]     (ScM)   {$Sc$};
 \node[con node, left=of RM, yshift= -1cm, xshift=1.9cm]    (TM)   {$T$};
 \node[mod node, below=of OM, yshift=-1cm]     (KM)   {$K$};
 \node[lib node, right=of OM, yshift = -0.5cm]     (GM)   {$G$};
 \node[lib node, right=of GM]     (SoM)   {$So$};
 \node[lib node, below=of GM, yshift=-0.3cm]     (BM)   {$B$};
 \node[lib node, below=of SoM, yshift=-0.2cm, xshift=0.3cm]     (StM)   {$St$};

\foreach \from/\to in {OM/RM, ScM/RM, TM/RM, SoM/StM,
BM/StM, KM/OM, RM/ScM, OM/ScM, KM/ScM, OM/KM, 
ScM/KM, TM/KM, StM/SoM, BM/SoM, RM/TM, OM/TM, KM/TM, StM/GM, BM/GM, StM/BM, OM/BM, SoM/BM} 
 \draw[-latex] (\from) edge [black](\to); 




\draw (TM) -- node [below=1 cm, xshift=1.5cm] {{\bf (b)\,\, Local\, AG \cite{GJ2017}}} (RM);

\node[mod node, right=of OM, xshift = 4.5cm] (OR) {$O$};
 \node[con node, left=of OR, yshift=-1cm, xshift=-0.2cm]     (RR)   {$R$};
 \node[con node, left=of RR, yshift= 1cm, xshift=0.6cm]     (ScR)   {$Sc$};
 \node[con node, left=of RR, yshift= -1cm, xshift=0.3cm]    (TR)   {$T$};
 \node[mod node, right=of RR, yshift= -1cm, xshift=-0.3cm]     (KR)   {$K$};
 \node[lib node, right=of OR, yshift = -0.5cm, xshift=0.3cm]     (SoR)   {$So$};
 \node[lib node, right=of SoR, yshift=-0.3cm]     (GR)   {$G$};
 \node[lib node, below=of SoR, yshift=-0.4cm]     (StR)   {$St$};
 \node[lib node, below=of GR, yshift=-0.1cm, xshift=0.2cm]     (BR)   {$B$};
 
\foreach \from/\to in {RR/StR,  StR/RR, StR/ScR, ScR/StR,  StR/SoR, SoR/StR, StR/GR, GR/StR, StR/BR, BR/StR, ScR/TR, TR/ScR, OR/KR, KR/OR, SoR/OR, OR/SoR}
\draw[-latex] (\from) edge [black](\to); 


\draw (TM) -- node [below=1 cm, xshift=7cm] {{\bf (c)\,\, Tree Structured Potential\, Game  \cite{GJ2016}}}  (RM);




%
%
%
%
%
%
%
%
\end{tikzpicture}

%% file: proofs.tex
\Large
\textbf{Strategic Prediction with Latent Aggregative Games\\\hspace*{3.5cm}(Supplementary Material)} \label{Supplementary}
\normalsize

\section{Convergence of dynamics}

We first provide some insight into our proof techniques.  We prove convergence to SNE via carefully crafted {\em Lyapunov functions} $\mathcal{V}$ that are locally positive definite and have a locally negative semidefinite time derivative, and thus satisfy the {\em Lyapunov stability} criterion.  The other proofs track the evolution of game dynamics around an equilibrium, where $\dot{q}_i = 0$ and $\dot{r}_i = 0$.  Specifically, we analyze conditions under which the Jacobian matrix of the linearization is {\em Hurwitz stable}, i.e., all the eigenvalues have negative real roots, and exploit the fact that the behavior of the ODE near equilibrium is same as its linear approximation when the real parts of all eigenvalues are non-zero. Our discrete updates would then converge to a Nash equilibrium with positive probability \cite{SA2005}. 

Recall that AA reveals more information about the evolution of neighbors' strategy. As a result, the PA settings, i.e. \eqref{DLAFP-True-Player} and \eqref{DLAGP-ODE_Player}, require additional subtle reasoning since at equilibrium $r_i^*$ converges only to $\mathcal{A}_i(q_{-i}^*)$ and not to $q_i^*$. Since $q_i$ evolves within $\Delta(A)$, stochasticity assumptions are required to ensure $r_i$ stays within the probability simplex as well. Note that the LAG-FP updates to strategies are smooth due to the entropy term (since $\tau > 0$), unlike LAG-GP.  Consequently, the results for LAG-GP require a separate treatment of completely mixed NE and strict NE, unlike LAG-FP where they can be analyzed without distinction.  Note that $\tau > 0$ ensures that best response is a singleton set and therefore we could leverage the ODE formulations. Differential inclusions \cite{BHS2005, BHS2006} could be used instead to handle $\tau = 0$. Our stability conditions can be simplified further when $\lambda$ is sufficiently large, whence the behavior may be understood solely in terms of $\gamma$. In general, via  standard eigenvalue arguments \cite{SA2005}, our protocols admit stable linearization under mild conditions.  

We now provide detailed proofs on convergence of dynamics. We use AA1, AA2 etc. to indicate that the result
pertains to convergence in an active aggregator setting. Likewise, we will use PA1 etc. for the passive aggregator setting. We start with the active aggregator. 

 \newtheorem{thmA}{Theorem}
\renewcommand\thethmA{AA\arabic{thmA}}

\begin{thmA}{\bf (Convergence under LAG-FP/AA to NE)}  \label{Thm1}
Let $(q_1^*, \ldots, q_n^*, z_1, \ldots, z_n)$ be a NE  under the dynamics in \eqref{DLAFP-True}.  There exists a matrix $\mathcal{D}$  such that the linearization of \eqref{DLAFP-True} with $\gamma > 0$ is locally asymptotically stable for $\lambda > 0$  if and only if the following matrix is stable
$$\begin{bmatrix}
    -I + (1 + \gamma \lambda) \mathcal{D} &  -\gamma \lambda \mathcal{D} \\
   \lambda I & -\lambda I
\end{bmatrix}.
$$
\end{thmA}
\begin{proof}
Since $\tau > 0$, best response is a singleton set, and the unique best response $\sigma_i^*$ can be obtained by setting the gradients of the payoff functions to 0. In particular, we have the best response  
\begin{equation} \label{BestResponse} \beta_{i}^{\tau}(\mathcal{A}_i(\sigma_{-i}), z_i) = \zeta\left(\dfrac{\sum_{j \neq i} w_{ij} \sigma_j - z_i}{\tau}\right) = \zeta\left(\dfrac{\mathcal{A}_i(\sigma_{-i}) - z_i}{\tau}\right), \end{equation}
where $\zeta$ is the softmax function with output coordinate $\ell$ given by  $$(\zeta(x))_{\ell} = \exp(x_{\ell})\bigg/\sum_{k} \exp(x_k).$$

Now recall from \eqref{DLAFP-True} that we have the following ODE:
\begin{eqnarray} 
\dot{q}_i & = & \underbrace{\beta_{i}^{\tau}(\mathcal{A}_i(q_{-i} + \gamma \dot r_{-i}), z_i^*)- q_i}_{\triangleq F_i(q_i, q_{-i}, r_{-i})}  \label{supp_eq1} \\
\dot{r}_i & = & \lambda(q_i - r_i).   \label{supp_eq2}
\end{eqnarray}

Since $\beta_{i}^{\tau}$maps it input to the simplex $\Delta(A)$, we note that the right side of \eqref{supp_eq1} is a difference between two probability distributions. Therefore this difference must sum to zero.  Moreover, since $|A| = m$, we have $m-1$ degrees of freedom that can be used to express this difference.  Therefore, we can investigate the evolution of $q_i$  via a matrix $N \in \mathbb{R}^{m \times (m-1)}$ of $(m-1)$ orthonormal columns such that 
$$N^{\top} N = I_{m-1}, \text{ and } 1_m^{\top} N = 0_{m-1},$$
where $I_{m-1}$ is the identity matrix of order $m-1$, and $1_m$ and $0_m$ are $m$-dimensional vectors with all coordinates set to 1 and 0 respectively.  We will sometimes omit the subscripts for $I_m$, $1_m$, and $0_m$ when the size will be clear from the context. The equilibrium $(q_i^*, q_{-i}^*)$ corresponds to a point $(q_i(t) = q_i^*, q_{-i}(t) = q_{-i}^*, r_i(t) = q_i^*, r_{-i}(t) = q_{-i}^*)$ of the dynamics.  It will be convenient to investigate the dynamics as the evolution of deviations around this point. Since $q_i$ is confined to $\Delta(A)$, we can express 
$$q_i(t) = q_i^* +  N \delta x_{q_i} (t),$$
where $\delta x_{q_i} (t) \in \mathbb{R}^{m-1}$ is uniquely specified, and likewise $r_i = q_i^* + \delta x_{r_i} (t)$ for some $\delta x_{r_i} (t)$.  Thus, we can define a block diagonal matrix $\mathcal{N} \in \mathbb{R}^{2nm \times 2n(m-1)}$, with each diagonal block set to $N$ and all other elements set to 0, such that 
\begin{eqnarray}  (q_1(t) - q_1^*, \ldots, q_n(t) - q_n^*,  ~~~r_{1}(t) - q_1^*,  \ldots, r_{n}(t) - q_n^*)^{\top} ~=~  \mathcal{N} \delta x(t)~, \label{supp_eq3} \end{eqnarray}
where $$\delta x(t) = (\delta x_{q_1} (t),  \ldots, \delta x_{q_n} (t),  \delta x_{r_1} (t), \ldots,  \delta x_{r_n} (t))^{\top} \in \mathbb{R}^{2n(m-1)}$$ is formed by stacking together the deviations at time $t$ in a column vector. Then, the following is immediate from \eqref{supp_eq3}:
\begin{equation}  \label{supp_eq4}
 \mathcal{N}^{\top} (q_1(t) - q_1^*, \ldots, q_n(t) - q_n^*,  r_{1}(t) - q_1^*,  \ldots, r_{n}(t) - q_n^*)^{\top}  ~=~  \mathcal{N}^{\top} \mathcal{N} \delta x(t) ~=~ \delta x(t). 
 \end{equation}

Denote the Jacobian matrix obtained by taking derivatives of vector $y$ with respect to vector $x$ by $J_{x} y$. We will linearize $\dot q_i = F_i(q_i, q_{-i}, r_{-i})$ in \eqref{supp_eq1} around  $\triangleq (q_1^*, q_{-1}^*, q_1^*, q_{-1}^*)$ using first order Taylor series.  Then, since $\dot q_i^* = 0$,  we note from \eqref{supp_eq1} and \eqref{supp_eq4} that
\begin{equation} \label{supp_eq5}
\dot \delta x_{q_i} = N^{\top} (\dot q_i-\dot q_i^*) =  N^{\top} \dot q_i(t) = N^{\top}  F_i(q_i, q_{-i}, r_{-i}).
\end{equation}
Now, at equilibrium, we have $\dot q_i = 0$ for all $i \in [n]$, and therefore we have from \eqref{supp_eq1} that
$$F_i(q_i^*, q_{-i}^*, r_{-i}^*) = 0_m. $$ Let ${\rm diag}(b)$ be a diagonal matrix with vector $b$ on the diagonal and all other elements set to 0. Ignoring the second order and higher terms, we therefore have by the Taylor series approximation that\\ \\

$\hspace*{3cm} F_i(q_i, q_{-i}, r_{-i})$
\begin{eqnarray*} 
 & \approx &   \sum_{k = 1}^n   J_{q_k} F_i(q_k, q_{-k}^*, q_{-k}^*) \bigg \vert_{q_k = q_k^*} (q_k - q_k^*)  ~+~  \sum_{k \neq i} J_{r_k} F_i(q_k^*, q_{-k}^*, q_{-ki}^*, r_{k}) \bigg \vert_{r_k = q_k^*}  (r_k - q_k^*) \\
& = & \sum_{k = 1}^n  J_{q_k} F_i(q_k, q_{-k}^*, q_{-k}^*) \bigg \vert_{q_k = q_k^*} N \delta x_{q_k}  ~+~  \sum_{k \neq i}  J_{r_k} F_i(q_k^*, q_{-k}^*, q_{-ki}^*, r_{k}) \bigg \vert_{r_k = q_k^*} N \delta x_{r_k}\\
& = &  - N \delta x_{q_i} +  \sum_{k  \neq i}  J_{q_k} F_i(q_k, q_{-k}^*, q_{-k}^*) \bigg \vert_{q_k = q_k^*} N \delta x_{q_k} ~+~ \sum_{k \neq i}  J_{r_k} F_i(q_k^*, q_{-k}^*, q_{-ki}^*, r_{k}) \bigg \vert_{r_k = q_k^*} N \delta x_{r_k} \\
& = & - N \delta x_{q_i} ~+~  (1+ \gamma \lambda) \sum_{k  \neq i} \tilde{D}_{ik} N \delta x_{q_k}   ~-~  \gamma \lambda   \sum_{k  \neq i} \tilde{D}_{ik} N \delta x_{r_k},   
\end{eqnarray*}
where $\tilde{D}_{ik} \triangleq  \dfrac{w_{ik}}{\tau}  \nabla \zeta \left( \dfrac{A_i(q_{-i}^*) - z_{i}}{\tau} \right)$, and $\nabla \zeta (b) \triangleq {\rm diag}(\zeta(b)) - \zeta(b) \zeta^{\top}(b)~.$ 

Define $D_{ik}  =  N^{\top} \tilde{D}_{ik}  N$. Since $N^{\top} N = I_{m-1}$, it follows immediately from \eqref{supp_eq5} that 
\begin{equation} \label{supp_eq6} \dot \delta x_{q_i} = - \delta x_{q_i} ~+~  (1+ \gamma \lambda) \sum_{k  \neq i} D_{ik} \delta x_{q_k} ~ - ~ \gamma \lambda \sum_{k  \neq i} D_{ik} \delta x_{r_k}. \end{equation}
Linearizing \eqref{supp_eq2}, we see that the Taylor approximation results in 
\begin{equation} \label{supp_eq7} \dot \delta x_{r_i} = \lambda (\delta x_{q_i} - \delta x_{r_i}). \end{equation}
We define 
$$\mathcal{D} = \begin{bmatrix}
    0       & D_{12} & D_{13} & \dots & D_{1n} \\
    D_{21}  & 0 & D_{23} & \dots & D_{2n} \\
    \vdots & \vdots & \vdots & \ddots & \vdots\\
    D_{n1}       & D_{n2} & D_{n3} & \dots & 0
\end{bmatrix}.
$$

Combining \eqref{supp_eq6} and \eqref{supp_eq7} together, we can write
$$\dot \delta x = \begin{bmatrix}
    -I + (1 + \gamma \lambda) \mathcal{D} &  -\gamma \lambda \mathcal{D} \\
   \lambda I & -\lambda I
\end{bmatrix}
\delta x.
$$
The statement of the theorem now follows immediately from the Hurwitz stability criterion. 
\end{proof}

%

 
\begin{thmA} {\bf (Convergence under LAG-GP/AA to CMNE)}  \label{Thm3}
Let $(q_1^*, \ldots, q_n^*, z_1, \ldots, z_n)$ be a completely mixed NE under the dynamics in  \eqref{DLAGP-ODE}.  Then the linearization of \eqref{DLAGP-ODE} with $\gamma > 0$ is locally asymptotically stable for $\lambda > 0$  if and only if the following matrix is stable
$$\begin{bmatrix}
     (1 + \gamma \lambda) W &  -\gamma \lambda W \\
   \lambda I & -\lambda I
\end{bmatrix}.
$$
\end{thmA}

\begin{proof}
Recall the ODE from  \eqref {DLAGP-ODE}:
\begin{eqnarray} 
\dot{q}_i & = & \Pi_{\Delta} [q_i + \mathcal{A}_i(q_{-i} + \gamma \dot r_{-i}) - z_i]- q_i   \label{3_supp_eq1}\\
\dot{r}_i & = & \lambda(q_i - r_i).      \label{3_supp_eq2}       
\end{eqnarray}
At equilibrium $(q_1^*, \ldots, q_n^*, z_1, \ldots, z_n),  \dot{q}_i = 0$ and $\dot{r}_i = 0$. Therefore, using \eqref{3_supp_eq1}, we have:
$$q_i^* = \Pi_{\Delta} [q_i^* + \mathcal{A}_i(q_{-i}^*) - z_i].$$
Since the equilibrium is completely mixed, $q_i^*$ is in the interior of $\Delta(A)$. We invoke Lemma 4.1 in \cite{SA2005} to get the following:
\begin{eqnarray}
NN^{\top} (\mathcal{A}_i(q_{-i}^*) - z_i)  & = & 0   \label{3_supp_eq3}\\
 \Pi_{\Delta} [q_i^* + \mathcal{A}_i(q_{-i}^*) - z_i + \delta y] - q_i^*  & = & NN^{\top}\bigg(\mathcal{A}_i(q_{-i}^*) \nonumber ~-~ z_i + \delta y\bigg), \label{3_supp_eq4}
\end{eqnarray}
for $\delta y$ sufficiently small, and $N$ as defined in the proof of Theorem \ref{Thm1}. Then, for a sufficiently small deviation $\delta x$, where $\delta x$ is as defined in Theorem 1, we get the following dynamics:
\begin{eqnarray} 
\dot{q}_i & = & NN^{\top} [\mathcal{A}_i(q_{-i} + \gamma \dot r_{-i}) - z_i]  \label{3_supp_eq5}\\
\dot{r}_i & = & \lambda(q_i - r_i).      \label{3_supp_eq6}       
\end{eqnarray} 
Linearizing these equations and noting that $N^{\top} N = I$, we get
\begin{eqnarray*}
\dot \delta x_{q_i}  & = &    N^{\top} \left(NN^{\top} (1 + \gamma \lambda) \sum_{k \neq i} w_{ik} N \delta x_{q_k} \right) ~-~  N^{\top} \left(NN^{\top} \gamma \lambda \sum_{k \neq i} w_{ik} N \delta x_{r_k} \right)\\
& = & (1 + \gamma \lambda)  N^{\top} \sum_{k \neq i} w_{ik} N \delta x_{q_k} ~-~  \gamma \lambda N^{\top} \sum_{k \neq i} w_{ik} N \delta x_{r_k}\\
& = & (1 + \gamma \lambda)  \sum_{k \neq i} w_{ik} \delta x_{q_k} ~-~  \gamma \lambda  \sum_{k \neq i} w_{ik} \delta x_{r_k},
\end{eqnarray*}
and
$$\dot \delta x_{r_i} = \lambda (\delta x_{q_i} - \delta x_{r_i}).$$
It follows immediately that
$$\dot \delta x = \begin{bmatrix}
    (1 + \gamma \lambda) W &  -\gamma \lambda W\\
   \lambda I & -\lambda I
\end{bmatrix}
\delta x,
$$
where the weight matrix
$$W = \begin{bmatrix}
    0       & w_{12} & w_{13} & \dots & w_{1n} \\
    w_{21}  & 0 & w_{23} & \dots & w_{2n} \\
    \vdots & \vdots & \vdots & \ddots & \vdots\\
    w_{n1}       & w_{n2} & w_{n3} & \dots & 0
\end{bmatrix}.
$$
\end{proof}

\newpage
\begin{thmA} {\bf (Convergence under LAG-GP/AA to SNE)} \label{Thm4}
Let $(q_1^*, \ldots, q_n^*, z_1, \ldots, z_n)$ be a strict NE under the dynamics in  \eqref{DLAGP-ODE}.  The associated equilibrium point $(q_i = q_i^*, q_{-i} = q_{-i}^*, r_i = q_i^*, r_{-i} = q_{-i}^*)$ is locally asymptotically stable for any $\gamma > 0$ and $\lambda > 0$. 
\end{thmA}
\begin{proof}
Recall the ODE from  \eqref {DLAGP-ODE}:
\begin{eqnarray} 
\dot{q}_i & = & \Pi_{\Delta} [q_i + \mathcal{A}_i(q_{-i} + \gamma \dot r_{-i}) - z_i]- q_i   \label{3_supp_eq1}\\
\dot{r}_i & = & \lambda(q_i - r_i).      \label{3_supp_eq2}       
\end{eqnarray}
To prove the local asymptotic stability of the ODE dynamics, we will define a Lyapunov function $\mathcal{V}$ that is locally positive definite and has locally negative semi-definite time derivative. Consider 
\begin{eqnarray} \label{Lyapunov1}
\mathcal{V}(q_i, q_{-i}, r_i, r_{-i}) \nonumber \\ \triangleq  \dfrac{1}{2} \sum_{i = 1}^n  \left((q_i - q_i^*)^{\top} (q_i - q_i^*) + \lambda (r_i - q_i)^{\top} (r_i - q_i) \right).
\end{eqnarray}
We define the shorthand $d_i \triangleq q_i + \mathcal{A}_i(q_{-i} + \gamma \dot r_{-i}) - z_i$.  Applying the chain rule, we see that the time derivative of $\mathcal{V}$,
\begin{eqnarray*}
\dot{\mathcal{V}} &  = & \sum_{i=1}^n \left(\dfrac{\partial \mathcal{V}}{\partial q_i} \right)^{\top} \dot q_i  + \sum_{i=1}^n \left(\dfrac{\partial \mathcal{V}}{\partial r_i} \right)^{\top} \dot r_i \nonumber \\
& = &  \sum_{i=1}^n \left[(q_i - q_i^*) + \lambda (q_i - r_i)\right]^{\top} \dot q_i \nonumber - \lambda^2  \sum_{i=1}^n (r_i - q_i)^{\top} (r_i - q_i) \\ &=&  \sum_{i=1}^n (q_i - q_i^*)^{\top} \dot q_i + \lambda  \sum_{i=1}^n  (q_i - r_i)^{\top} \dot q_i \nonumber  ~-~   \lambda^2 \sum_{i=1}^n ||r_i - q_i||^2  \\ &=&  \sum_{i=1}^n (q_i - q_i^*)^{\top}  \Pi_{\Delta}(d_i) - \sum_{i=1}^n (q_i - q_i^*)^{\top} q_i \nonumber  ~+~ \lambda  \sum_{i=1}^n  (q_i - r_i)^{\top} \dot q_i -   \lambda^2 \sum_{i=1}^n ||r_i - q_i||^2.
\end{eqnarray*}
Also, we note that
\begin{eqnarray*}
\sum_{i=1}^n ||\dot{q_i}||^2  &  = &  \sum_{i=1}^n ||\Pi_{\Delta}(d_i) - q_i||^2\\
& = &  \sum_{i=1}^n ||\Pi_{\Delta}(d_i)||^2 + \sum_{i=1}^n q_i^{\top} q_i - 2 \sum_{i=1}^n q_i^{\top} \Pi_{\Delta}(d_i). 
\end{eqnarray*}
This immediately implies 
\begin{eqnarray} \label{Lyapunov1Derivative}
 \dot{\mathcal{V}} ~+~ \sum_{i=1}^n ||\dot{q_i}||^2 &  = &  \sum_{i=1}^n \underbrace{\left(\Pi_{\Delta}(d_i) - q_i^* \right)^{\top} \left(\Pi_{\Delta}(d_i) - q_i \right)}_{(B)} \nonumber ~~+~~ \lambda  \sum_{i=1}^n  (q_i - r_i)^{\top} \dot q_i  -   \lambda^2 \sum_{i=1}^n ||r_i - q_i||^2. 
\end{eqnarray}
Consider $(B) =  \left(\Pi_{\Delta}(d_i) - q_i^* \right)^{\top} \left(\Pi_{\Delta}(d_i) - q_i \right)$. Since $\Delta(A)$ is a convex set, the projection property implies
$$[\Pi_{\Delta}(d_i)]^{\top}  \left(\Pi_{\Delta}(d_i) - q_i \right) \leq d_i^{\top} \left(\Pi_{\Delta}(d_i) - q_i \right),$$
whence we note
\begin{eqnarray*}
(B) & = & \left(\Pi_{\Delta}(d_i) - q_i^* \right)^{\top} \left(\Pi_{\Delta}(d_i) - q_i \right)\\
& = & [\Pi_{\Delta}(d_i)]^{\top}  \left(\Pi_{\Delta}(d_i) - q_i \right) - \left(\Pi_{\Delta}(d_i) - q_i \right)^{\top} q_i^*\\
& \leq & d_i^{\top} \left(\Pi_{\Delta}(d_i) - q_i \right) -  \left(\Pi_{\Delta}(d_i) - q_i \right)^{\top} q_i^*\\
& = & (d_i - q_i^*)^{\top} \left(\Pi_{\Delta}(d_i) - q_i \right)\\
& = & (q_i + \mathcal{A}_i(q_{-i} + \gamma \dot r_{-i}) - z_i - q_i^*)^{\top} \left(\Pi_{\Delta}(d_i) - q_i \right).
\end{eqnarray*}
Now, we note from the definition of $\mathcal{V}$ in \eqref{Lyapunov1} that by decreasing the distances $(q_i - r_i)$ and  $(q_i - q_i^*)$, we can make $\mathcal{V}(q_i, q_{-i}, r_i, r_{-i})$ arbitrarily close to 0 from the right. In other words, we can consider a sufficiently small neighborhood around the equilibrium such that as $r_i, q_i \to  q_i^*$, (B) tends to
\begin{eqnarray*}
& (q_i^* + \mathcal{A}_i(q_{-i^*} + \delta y) - z_i - q_i^*)^{\top} \left(\Pi_{\Delta}(d_i) - q_i^* \right)\\
= & (\mathcal{A}_i(q_{-i^*} + \delta y) - z_i)^{\top} \left(\Pi_{\Delta}(d_i) - q_i^* \right), \\
= & \left(\Pi_{\Delta}(d_i) - q_i^* \right)^{\top} \dfrac{\partial U_i(q_i, q_{-i}^* + \delta_y, z_i)}{\partial q_i}\bigg \vert_{q_i = q_i^*}
< & \hspace*{-0.2cm}0
\end{eqnarray*}
for some sufficiently small $\delta y$ and $\Pi_{\Delta}(d_i) \neq q_i^*$.  The last inequality follows since $(q_1^*, \ldots, q_n^*, z_1, \ldots, z_n)$ is a strict Nash equilibrium, whereby (a) $(q_i^*, q_{-i}^*)$ is a pure strategy Nash equilibrium (since $\mathcal{A}_i(\cdot)$ is a linear transformation and the payoff maximization happens at the vertex), the  and  (b) $q_i^*$ is a (strictly) best response to $q_{-i}^*$ and nearby strategies.  Therefore, we see from \eqref{Lyapunov1Derivative} that for a sufficiently small neighborhood around the equilibrium point, 
\begin{eqnarray*}
\dot{\mathcal{V}} & \leq &  - \sum_{i=1}^n ||\dot{q_i}||^2  + \lambda  \sum_{i=1}^n  (q_i - r_i)^{\top}\dot q_i -   \lambda^2 \sum_{i=1}^n ||r_i - q_i||^2\\
& = &  -~ \sum_{i=1}^n \left(||\dot{q_i}||^2  +  \lambda^2  ||r_i - q_i||^2  \right) +  \lambda \sum_{i=1}^n  (q_i - r_i)^{\top}\dot q_i\\
& \leq & -~ \sum_{i=1}^n \left(||\dot{q_i}||^2  +  \lambda^2  ||r_i - q_i||^2  \right)  ~+~ \dfrac{1}{2} \sum_{i=1}^n \left(||\dot{q_i}||^2  +  \lambda^2  ||r_i - q_i||^2  \right)\\
& = &  -~ \dfrac{1}{2} \sum_{i=1}^n \left(||\dot{q_i}||^2  +  \lambda^2  ||r_i - q_i||^2  \right),
\end{eqnarray*}
where we have invoked the Cauchy-Schwarz inequality in the penultimate line. Since this quantity is non-positive, we see that $\dot{\mathcal{V}}$ is locally negative semi-definite. Finally, it is clear from \eqref{Lyapunov1} that $\mathcal{V}(q_i, q_{-i}, r_i, r_{-i}) > 0$ in the neighborhood $(q_i, q_{-i}, r_i, r_{-i})$ of the equilibrium point $(q_i = q_i^*, q_{-i} = q_{-i}^*, r_i = q_i^*, r_{-i} = q_{-i}^*)$, and $\mathcal{V}(q_i^*, q_{-i}^*, q_i^*, q_{-i}^*) = 0$. Thus,  $\mathcal{V}$ is locally positive definite, and the statement of the theorem follows. 
\end{proof}

We will now characterize conditions for convergence in the passive aggregator setting. 

\newtheorem{thmP}{Theorem}
\renewcommand\thethmP{PA\arabic{thmP}}

\begin{thmP} {\bf (Convergence under LAG-FP/PA to NE)}  \label{Thm2}
Let the weight matrix $W$ be stochastic. Let $(q_1^*, \ldots, q_n^*, z_1, \ldots, z_n)$ be a NE under the dynamics in  \eqref{DLAFP-True-Player}.  There exists a matrix $\mathcal{D}_1$ with zero diagonal, and a block diagonal matrix $\mathcal{D}_2$ such that the linearization of \eqref{DLAFP-True-Player} with $\gamma > 0$ is locally asymptotically stable for $\lambda > 0$  if and only if if the following matrix is stable
$$\begin{bmatrix}
    -I + (1 + \gamma \lambda) \mathcal{D}_1 &  -\gamma \lambda \mathcal{D}_2\\
   \lambda W & -\lambda I
\end{bmatrix}.
$$
\end{thmP}
\begin{proof}
We reproduce the ODE from \eqref{DLAFP-True-Player}:
\begin{eqnarray} 
\dot{q}_i & = & \beta_{i}^{\tau}(\mathcal{A}_i(q_{-i}) + \gamma \dot r_{i}), z_i)- q_i  \label{2_supp_eq1}\\
\dot{r}_i & = & \lambda(\mathcal{A}_i(q_{-i}) - r_i).  \label{2_supp_eq2}
\end{eqnarray}
Note that at equilibrium $\dot{r}_i = 0$, but unlike Theorem \ref{Thm1}, $r_i$ does not converge to $q_i^*$.  Specifically, we note that the equilibrium $(q_i^*, q_{-i}^*)$ corresponds to a point $(q_i(t) = q_i^*, q_{-i}(t) = q_{-i}^*, r_i(t) = \mathcal{A}_i(q_{-i}^*)), i \in [n]$, of the dynamics. Therefore, we will instead linearize around this point. Since the weight matrix $W$ is stochastic, we must have $\mathcal{A}_i(q_{-i}^*) \in \Delta(A)$. Therefore,  we can investigate the deviation of $r_i$ around $\mathcal{A}_i(q_{-i}^*)$ with the help of matrix $N$ defined in Theorem \ref{Thm1}.  In particular, we can express the deviation vector $\delta x = (\delta x_{q_1} , \ldots, \delta x_{q_n}, \delta x_{r_1} , \ldots, \delta x_{r_n})^{\top}$ as:
\begin{equation} \label{2_supp_eq3}  \bigg(q_1(t) - q_1^*, \ldots, q_n(t) - q_n^*,  ~ r_{1}(t) -  \mathcal{A}_1(q_{-1}^*),  \ldots, r_{n}(t) - \mathcal{A}_n(q_{-n}^*)\bigg)^{\top} = \mathcal{N} \delta x(t), \end{equation}
where the block diagonal matrix $\mathcal{N}$ is as defined in Theorem \ref{Thm1}.  Linearizing around our equilibrium point and proceeding similarly to Theorem \ref{Thm1},  
we get
\begin{equation} \label{2_supp_eq6} \dot \delta x_{q_i} = - \delta x_{q_i}  ~+~  (1+ \gamma \lambda) \sum_{k  \neq i} D_{ik} \delta x_{q_k} ~ - ~ \gamma \lambda  C_{i} \delta x_{r_i}. \end{equation} 
where  
$$D_{ik} \triangleq  \dfrac{w_{ik}}{\tau} N^{\top} \nabla \zeta \left( \dfrac{A_i(q_{-i}^*) - z_{i}}{\tau} \right) N, $$
$$C_{i} \triangleq \dfrac{1}{\tau}  N^{\top} \nabla \zeta \left( \dfrac{A_i(q_{-i}^*) - z_{i}}{\tau} \right) N, $$
and
$$\nabla \zeta (b) \triangleq {\rm diag}(\zeta(b)) - \zeta(b) \zeta^{\top}(b),$$
with $\zeta(b)$ the same as in Theorem \ref{Thm1}.  
Additionally, we have
\begin{equation} \label{2_supp_eq7} \dot \delta x_{r_i} = \lambda \sum_{k \neq i} w_{ik}  \delta x_{q_i} - \lambda \delta x_{r_i}. \end{equation}
Recall that the weight matrix 
$$W = \begin{bmatrix}
    0       & w_{12} & w_{13} & \dots & w_{1n} \\
    w_{21}  & 0 & w_{23} & \dots & w_{2n} \\
    \vdots & \vdots & \vdots & \ddots & \vdots\\
    w_{n1}       & w_{n2} & w_{n3} & \dots & 0
\end{bmatrix}.
$$
Define 
$$\mathcal{D}_1 \triangleq \begin{bmatrix}
    0       & D_{12} & D_{13} & \dots & D_{1n} \\
    D_{21}  & 0 & D_{23} & \dots & D_{2n} \\
    \vdots & \vdots & \vdots & \ddots & \vdots\\
    D_{n1}       & D_{n2} & D_{n3} & \dots & 0
\end{bmatrix}, ~~\text{ and }$$

$$\mathcal{D}_2 \triangleq \begin{bmatrix}
    C_1     & 0 & 0 & \dots & 0 \\
    0  & C_2 & 0 & \dots & 0 \\
    \vdots & \vdots & \vdots & \ddots & \vdots\\
    0       & 0 & 0 & \dots & C_n
\end{bmatrix}.
$$
Then, the proof follows by combining \eqref{2_supp_eq6} and \eqref{2_supp_eq7}, since we can express the deviations as
$$\dot \delta x = \begin{bmatrix}
    -I + (1 + \gamma \lambda) \mathcal{D}_1 &  -\gamma \lambda \mathcal{D}_2\\
   \lambda W & -\lambda I
\end{bmatrix}
\delta x.
$$ \end{proof}

\begin{thmP} {\bf (Convergence under LAG-GP/PA to CMNE)} \label{Thm5}
Let the weight matrix $W$ be stochastic. Let $(q_1^*, \ldots, q_n^*, z_1, \ldots, z_n)$ be a completely mixed NE under the dynamics in \eqref{DLAGP-ODE_Player}.  Then the linearization of \eqref{DLAGP-ODE_Player} with $\gamma > 0$ is locally asymptotically stable for $\lambda > 0$   if and only if the following matrix is stable
$$\begin{bmatrix}
     (1 + \gamma \lambda) W &  -\gamma \lambda W \\
    \lambda W & -\lambda I
\end{bmatrix}.
$$
\end{thmP}
\begin{proof}
Recall the ODE from  \eqref {DLAGP-ODE_Player}:
\begin{eqnarray}
\dot{q}_i & = & \Pi_{\Delta} [q_i + \mathcal{A}_i(q_{-i}) + \gamma \dot r_{i} - z_i]- q_i   \label{5_supp_eq1} \\
\dot{r}_i & = & \lambda(\mathcal{A}_i(q_{-i}) - r_i).  \label{5_supp_eq2}
\end{eqnarray}
At equilibrium $(q_1^*, \ldots, q_n^*, z_1, \ldots, z_n),  \dot{q}_i = 0$ and $\dot{r}_i = 0$. Therefore, using \eqref{5_supp_eq1}, we have:
$$q_i^* = \Pi_{\Delta} [q_i^* + \mathcal{A}_i(q_{-i}^*) - z_i].$$
Proceeding along the lines of proof of Theorem \ref{Thm3}, for a sufficiently small deviation $\delta x$ as defined in Theorem \ref{Thm2}, we can equivalently analyze the following dynamics:
\begin{eqnarray*} 
\dot{q}_i & = & NN^{\top} [\mathcal{A}_i(q_{-i} )+ \gamma \dot r_{-i} - z_i]  \label{5_supp_eq5}\\
\dot{r}_i & = & \lambda(q_i - r_i).      \label{5_supp_eq6}       
\end{eqnarray*} 
Linearizing these equations and noting $N^{\top} N = I$, we get
\begin{eqnarray*}
 \dot \delta x_{q_i} &\hspace*{-0.3cm} ~=~ &   \hspace*{-0.3cm} N^{\top} \left(NN^{\top} (1 + \gamma \lambda) \sum_{k \neq i} w_{ik} N \delta x_{q_k} \right) ~-~ N^{\top} \left(NN^{\top} \gamma \lambda \sum_{k \neq i} N \delta x_{r_k} \right)\\
& \hspace*{-0.3cm} = & \hspace*{-0.3cm} (1 + \gamma \lambda)  N^{\top} \sum_{k \neq i} w_{ik} N \delta x_{q_k} -  \gamma \lambda N^{\top} \sum_{k \neq i} w_{ik} N \delta x_{r_k}\\
& \hspace*{-0.3cm} = & \hspace*{-0.3cm} (1 + \gamma \lambda)  \sum_{k \neq i} w_{ik} \delta x_{q_k}  -  \gamma \lambda  \sum_{k \neq i} w_{ik} \delta x_{r_k},
\end{eqnarray*}
and
$$\dot \delta x_{r_i} ~~=~~ \lambda \sum_{k \neq i} w_{ik}  \delta x_{q_i} - \lambda \delta x_{r_i}.$$
It follows immediately that
$$\dot \delta x = \begin{bmatrix}
    (1 + \gamma \lambda) W &  -\gamma \lambda W\\
    \lambda W & -\lambda I
\end{bmatrix}
\delta x,
$$
where the weight matrix
$$W = \begin{bmatrix}
    0       & w_{12} & w_{13} & \dots & w_{1n} \\
    w_{21}  & 0 & w_{23} & \dots & w_{2n} \\
    \vdots & \vdots & \vdots & \ddots & \vdots\\
    w_{n1}       & w_{n2} & w_{n3} & \dots & 0
\end{bmatrix}.
$$
\end{proof}

\begin{thmP} {\bf (Convergence under LAG-GP/PA to SNE)} \label{Thm6}
Let the weight matrix $W$ be doubly stochastic. Let $(q_1^*, \ldots, q_n^*, z_1, \ldots, z_n)$ be a strict NE under the dynamics in \eqref{DLAGP-ODE_Player}.  The associated equilibrium point $(q_i = q_i^*, r_i = A_i(q_{-i}^*))_{i \in [n]}$ is locally asymptotically stable for sufficiently small $\gamma \lambda$, where $\gamma > 0$ and $\lambda > 0$. 
\end{thmP}

\begin{proof}
Recall the ODE  from  \eqref {DLAGP-ODE_Player}:
\begin{eqnarray*}
\dot{q}_i & = & \Pi_{\Delta} [q_i + \mathcal{A}_i(q_{-i}) + \gamma \dot r_{i} - z_i]- q_i  \\
\dot{r}_i & = & \lambda(\mathcal{A}_i(q_{-i}) - r_i).  
\end{eqnarray*}
 We will prove local asymptotic stability via a Lyapunov function $\mathcal{V}$ that is locally positive definite and has locally negative semi-definite time derivative. Consider 
\begin{eqnarray} \label{Lyapunov2}
\mathcal{V}(q_i, q_{-i}, r_i, r_{-i})  \triangleq  \dfrac{1}{2} \sum_{i = 1}^n  \bigg((q_i - q_i^*)^{\top} (q_i - q_i^*)  ~+~~ \lambda \left(r_i -\mathcal{A}_i( q_{-i})\right)^{\top} \left(r_i - \mathcal{A}_i( q_{-i})\right) \bigg).
\end{eqnarray}
We define the shorthand $\tilde{d}_i \triangleq q_i + \mathcal{A}_i(q_{-i}) + \gamma \dot r_{-i} - z_i$.  Applying the chain rule, we see that the time derivative of $\mathcal{V}$,
\begin{eqnarray*}
\dot{\mathcal{V}} &  = & \sum_{i=1}^n \left(\dfrac{\partial \mathcal{V}}{\partial q_i} \right)^{\top} \dot q_i  + \sum_{i=1}^n \left(\dfrac{\partial \mathcal{V}}{\partial r_i} \right)^{\top} \dot r_i \nonumber \\
& = & \sum_{i=1}^n \left[(q_i - q_i^*) - \lambda \sum_{k \neq i} w_{ki} (r_k - \mathcal{A}_{k} \left(q_{-k}) \right)\right]^{\top} \dot q_i  ~-~  \lambda^2  \sum_{i=1}^n \left(r_i -\mathcal{A}_i(q_{-i})\right)^{\top}\left(r_i - \mathcal{A}_i( q_{-i})) \right)\nonumber   \\
& = & \sum_{i=1}^n (q_i - q_i^*)^{\top} \dot q_i -  \lambda^2 \sum_{i=1}^n ||r_i - \mathcal{A}_i(q_{-i})||^2  ~-~    \lambda \sum_{i=1}^n \left(\sum_{k \neq i} w_{ki} (r_k - \mathcal{A}_{k}  \left(q_{-k}) \right)\right)^{\top} \dot q_i~~.
\end{eqnarray*}
Also, we note that
\begin{eqnarray*}
\sum_{i=1}^n ||\dot{q_i}||^2  &  = &  \sum_{i=1}^n ||\Pi_{\Delta}(\tilde{d}_i) - q_i||^2 \\
& = &  \sum_{i=1}^n ||\Pi_{\Delta}(\tilde{d}_i)||^2 + \sum_{i=1}^n q_i^{\top} q_i - 2 \sum_{i=1}^n q_i^{\top} \Pi_{\Delta}(\tilde{d}_i). 
\end{eqnarray*}
Proceeding along the lines of Theorem \ref{Thm4}, for sufficiently small $\gamma \lambda$ and sufficiently small neighborhood around the equilibrium point, we can show that
\begin{eqnarray*}
\dot{\mathcal{V}} & \leq & - \sum_{i=1}^n ||\dot{q_i}||^2  + \lambda \sum_{i=1}^n \left(\sum_{k \neq i} w_{ki} (\mathcal{A}_{k} \left(q_{-k}) - r_k \right)\right)^{\top} \dot q_i 
 ~-~   \lambda^2 \sum_{i=1}^n ||r_i - \mathcal{A}_i(q_{-i})||^2\\
& = & - \sum_{i=1}^n \left(||\dot{q_i}||^2  +  \lambda^2  ||r_i - \mathcal{A}_i(q_{-i})||^2  \right)  ~+~  \sum_{i=1}^n \sum_{k \neq i} w_{ki} \left(\lambda  (\mathcal{A}_{k} \left(q_{-k}) - r_k \right)^{\top} \dot q_i\right)\\
& \leq &  - \sum_{i=1}^n \left(||\dot{q_i}||^2  +  \lambda^2  ||r_i - \mathcal{A}_i(q_{-i})||^2  \right) ~+~ \dfrac{1}{2} \sum_{i=1}^n \sum_{k \neq i} w_{ki} \left(\lambda^2 ||r_k - \mathcal{A}_k(q_{-k})||^2 + ||\dot q_i||^2 \right)
\end{eqnarray*}
by noting that $w_{ki} \geq 0$ for all $i \in [n]$, $k \neq i$ and invoking Cauchy-Schwarz. Furthermore, since $W$ is doubly stochastic, we have $\sum_{k \neq i} w_{ki} = 1$ and $\sum_{k \neq i} w_{ik} = 1$ for all $i \in [n]$. Thus, we may decompose the second term on the right in the last equation as
\begin{eqnarray*}
& \dfrac{1}{2} \displaystyle \sum_{i=1}^n \sum_{k \neq i} w_{ki} \left(\lambda^2 ||r_k - \mathcal{A}_k(q_{-k})||^2 + ||\dot q_i||^2 \right) \\
= & \dfrac{\lambda^2}{2}  \displaystyle \sum_{i=1}^n \sum_{k \neq i} w_{ki}  ||r_k - \mathcal{A}_k(q_{-k})||^2 +  \dfrac{1}{2} \displaystyle \sum_{i=1}^n ||\dot q_i||^2 \sum_{k \neq i} w_{ki} \\
=  & \dfrac{\lambda^2}{2}  \displaystyle \sum_{i=1}^n \sum_{k \neq i} w_{ki}  ||r_k - \mathcal{A}_k(q_{-k})||^2 +  \dfrac{1}{2} \displaystyle \sum_{i=1}^n ||\dot q_i||^2. \\ 
\end{eqnarray*}
The first term in the last equation may be interpreted as a weighted outgoing flow from player $i$ to player $k \neq i$. Now viewing this from the {\em dual} perspective of the total incoming flow, we equivalently have
\begin{eqnarray*}
\dot{\mathcal{V}} & \leq &  - \sum_{i=1}^n \left(||\dot{q_i}||^2  +  \lambda^2  ||r_i - \mathcal{A}_i(q_{-i})||^2  \right) ~+~  \dfrac{\lambda^2}{2}  \displaystyle \sum_{i=1}^n \sum_{k \neq i} w_{ik}  ||r_i - \mathcal{A}_i(q_{-i})||^2 +  \dfrac{1}{2} \displaystyle \sum_{i=1}^n ||\dot q_i||^2 \\
& = & - \sum_{i=1}^n \left(||\dot{q_i}||^2  +  \lambda^2  ||r_i - \mathcal{A}_i(q_{-i})||^2  \right) ~+~  \dfrac{\lambda^2}{2}  \displaystyle \sum_{i=1}^n ||r_i - \mathcal{A}_i(q_{-i})||^2 \sum_{k \neq i} w_{ik}  +   \dfrac{1}{2} \displaystyle \sum_{i=1}^n ||\dot q_i||^2 \\
& = & - \sum_{i=1}^n \left(||\dot{q_i}||^2  +  \lambda^2  ||r_i - \mathcal{A}_i(q_{-i})||^2  \right)   ~+~ \dfrac{1}{2} \displaystyle \sum_{i=1}^n  \left(\lambda^2  ||r_i - \mathcal{A}_i(q_{-i})||^2   +   ||\dot q_i||^2 \right)\\
& = & - \dfrac{1}{2} \displaystyle \sum_{i=1}^n  \left(\lambda^2  ||r_i - \mathcal{A}_i(q_{-i})||^2   +   ||\dot q_i||^2 \right)\\
& \leq & 0,
\end{eqnarray*}
which implies that $\dot{\mathcal{V}}$ is locally negative semi-definite. The local positive definiteness of $\mathcal{V}$ may be argued along the same lines as the proof of Theorem \ref{Thm4} and we are done.  
\end{proof}

\section{Identifiability of one-shot LAGs}
We now present the results on provably recovering the structure of one-shot LAGs from data. Specifically, we characterize the conditions under which LAGs with one step dynamics become identifiable, and provide an algorithm to recover the structure of the game, i.e., the neighbors for each player $i \in [n]$ with the signs (positive or negative) of their respective influences. 


Our recovery procedure adapts the primal-dual witness method \cite{W2009} for structure estimation in games. The method has previously been applied in several non-strategic settings such as Lasso \cite{W2009} and Ising models \cite{RWL2010}. Recently, \cite{GH2017} employed this method to recover a set of pure strategy Nash equilibria (PSNE) from data consisting of a subset of PSNE, and a small fraction of non-equilibrium outcomes assumed to be sampled under their noise models in the setting of linear influence games. However, the problem of structure recovery is significantly harder: it is known \cite{HO2015, GH2017} that the problem becomes non-identifiable in the setting of PSNE, since multiple game  structures may pertain to the same of PSNE. To our knowledge, there are no known results on the provable recovery of structure of graphical games from data. We fill this gap by characterizing conditions under which one-shot LAGs become identifiable. 

Our approach follows the general  proof structure of primal-dual witness method in the context of model selection for Ising models \cite{RWL2010}. However, our setting is significantly different from the  setting in \cite{RWL2010} where context and dynamics play no part, and all the observed data is assumed to be sampled from a common (global) distribution expressible in a closed form. In contrast, each observed outcome in our setting is sampled from a separate joint strategy profile following one-step of dynamics initiated under a different context.     

Specifically, in the one-shot setting, consider a dataset $D = \{(x^{(m)}, a^{(m)}) \in \mathcal{X} \times \mathcal{Y}, m \in [M]\}$ where $a^{(m)}$ is the action profile (i.e. observed outcome) sampled from the joint player strategies after one round of communication. Assume that the type parameters $\theta = (\theta_1, \ldots, \theta_n)$ are known. Then, since types for any context are determined by the parameters $\theta$, we have access to the player types $z^{(m)}(x^{(m)}) = (z_1^{(m)}, \ldots, z_n^{(m)})$, which in turn determine determine the initial strategies for all the examples $m \in [M]$. 
We focus on binary actions here since they let us simplify the exposition while conveying the essential ideas. Specifically, each player $i \in [n]$ initially plays action 1 with probability
$$\phi_i^{(m)} = \xi(z_i^{(m)}) \triangleq \dfrac{1}{1 + \exp(-z_i^{(m)})}~,$$
and the action 0 with probability $1 - \phi_i^{(m)}$. We define $\phi^{(m)} = (\phi_1^{(m)}, \ldots, \phi_n^{(m)})$, and $\Phi_{-i}^{(m)} ~~=~~ (\phi_j^{(m)})_{j \neq i}$. We focus on the gradient update setting where after one round of communication, player $i$ responds to its neighbors with its updated strategy $(\sigma_{i}^{*(m)}, 1 - \sigma_{i}^{*(m)})$, where
$$\sigma_{i}^{*(m)} \triangleq  \xi\left(\phi_i^{(m)} + \alpha(\sum_{j \neq i} w_{ij}^* \phi_j^{(m)} - z_i^{(m)})\right)~,$$
such that $\alpha > 0$, and $w_{ij}^* \in \mathbb{R}$ is the true influence (i.e. interaction weight) of player $j \in [n]\setminus\{i\}$ on $i$. Recall that we call player $j$ a neighbor of $i$ if $|w_{ij}^*| > 0$.  Finally, action $a_i^{(m)}$ is sampled from 
the updated strategy, and we obtain the joint profile 
$a^{(m)} = \{a_i^{(m)}, i \in [n]\}$ as the observed outcome. Our goal is to estimate, from $D$ and $\alpha$, the {\em support} $S_i$, or the set of neighbors $j$ for $i$, i.e., the players that have influence $w_{ij}^* \neq 0$. We can thus separate the influence of neighbors of $i$ from the non-neighbors by defining the set of non-zero weights
$w_{i, S}^* = \{w_{ij}^*| j \in S_i\}$. We denote the complement of a set $A$ by $A^c$. Thus, $w_{ij}^* = 0$ for $j \in S_i^c$. We equivalently write $w_{i, S^c}^* = \boldsymbol{0}$. We are interested in recovering not only the support of each player $i$, but also the correct sign of influence (i.e. positive or negative) of each neighbor $j$ on $i$.
We consider the average cross-entropy loss between the strategy under $w_i$ and the observed outcome.   
\begin{equation}  \label{loss} 
\ell_i(w_i; D)  ~=~ \dfrac{1}{M} \sum_{m=1}^{M} - \left(a_i^{(m)} \log(\sigma_i^{(m)}) +  (1 - a_i^{(m)}) \log(1 - \sigma_i^{(m)}) \right)~.
\end{equation}
We compute the gradient and the Hessian of the sample loss:
\begin{equation} \label{gradient}
\nabla \ell_i(w_i; D)  ~=~  \dfrac{\alpha}{M}  \sum_{m=1}^M  (\sigma_i^{(m)} - a_i^{(m)})~ \Phi_{-i}^{(m)}~,
\end{equation} 
\begin{equation} \label{Fisher}
H_{i}^M ~\triangleq~ \nabla^2 \ell_i(w_i; D)  ~=~ \dfrac{\alpha^2}{M} \sum_{m=1}^M \sigma_i^{(m)} (1 - \sigma_i^{(m)})~ \Phi_{-i}^{(m)} \Phi_{-i}^{(m)^\top}~.
\end{equation}
We will often use the variance function $\eta_i(w_i; m) \triangleq \alpha^2 \sigma_i^{(m)} (1 - \sigma_i^{(m)})$ as a shorthand, and write
\begin{equation} \label{FisherShort} H_i^M = \dfrac{1}{M} \sum_{m=1}^M \eta_i(w_i; m) ~ \Phi_{-i}^{(m)} \Phi_{-i}^{(m)^\top}~.
\end{equation}
We denote by $H_{i, SS}^{*M}$ the submatrix obtained by restricting the Hessian $H_{i}^{*M}$, pertaining to true weights,  to rows and columns corresponding to neighbors, i.e., players in $S_i$. Likewise, $H_{i, SS^c}^{*M}$ denotes the submatrix restricted to rows pertaining to $S_i$ (neighbors) and columns to $S^c_i$ (non-neighbors). 

We will provide detailed analysis under sample Fisher matrix assumptions. We will omit the analysis for the population setting that can be derived by imposing analogous assumptions directly on the population matrices, and making concentration arguments that show these assumptions hold in the sampled setting with high probability. 
Recall from the main text that we make the following assumptions that are reminiscent of those for support recovery under Lasso \cite{W2009}, and model selection in Ising models \cite{RWL2010}. We first recall our assumptions from the main text. 

\textbf{Assumptions.}
\begin{equation} \label{min_eigen}
\Lambda_{\min}\left(H_{i, SS}^{*M}\right) ~\geq~ \alpha^2 C_{min}~. 
\end{equation}
\begin{equation} \label{max_eigen}
\Lambda_{\max}\left(\dfrac{1}{M} \sum_{m=1}^M \Phi_{-i}^{(m)} \Phi_{-i}^{(m)^\top}\right) ~\leq~ C_{max}~. 
\end{equation}
\begin{equation} \label{incoherence}
|||H^{*M}_{i, S^cS} (H^{*M}_{i, SS})^{-1}|||_\infty ~\leq~ 1 - \gamma~,
\end{equation}
such that $C_{\min} > 0$, $C_{\max} < \infty$, and $\gamma \in (0, 1]$. In our notation, $|||A|||_\infty$ denotes the maximum $\ell_1$ norm across rows of matrix $A$, and $|||A|||_2$ denotes the spectral norm (i.e. maximum singular value) of $A$.
$\Lambda_{min}(A)$ and  $\Lambda_{max}(A)$ refer respectively to the minimum and the maximum eigenvalue of a square matrix $A$.

\textbf{Analysis.}
We propose to solve the following regularized problem for each player $i \in [n]$ separately. 
\begin{equation} \label{logistic}
\arg\!\!\min_{w_i \in \mathbb{R}^{n-1}} \ell_i(w_i; D) ~+~ \lambda_{M, n, d} ||w_i||_1~, 
\end{equation}
where $\lambda_{M, n, d} > 0$ is a regularization parameter that depends on the sample size $M$, the number of players $n$, and the maximum {\em degree} (i.e. number of neighbors) of any player. For brevity, we will omit the dependence of this parameter on $n$ and $d$, and simply write $\lambda_{M}$. This problem is convex but not differentiable everywhere because of the $\ell_1$ penalty. Note that since the problem is not strictly convex, it might have multiple minimizing solutions. 
For any such optimal solution $\hat{w}_i$, we must have by KKT conditions, 
\begin{equation} \label{KKT} \nabla \ell_i(\hat{w}_i; D) ~+~ \lambda_M \hat{\kappa}_i ~=~ \boldsymbol{0}~, \end{equation}
where the subgradient $\hat{\kappa}_i \in \mathbb{R}^{n-1}$ is such that
\begin{equation} \label{Subgradient}
\hat{\kappa}_{ij} = \text{sign}(\hat{w}_{ij}) \in \{\pm 1\} \text{  if  } \hat{w}_{ij} \neq 0, \text{ and } |\hat{\kappa}_{ij}| \leq 1 \text{ otherwise}.
\end{equation}
We would like to ensure the following conditions in order to recover the signed neighborhood for $i$.
\begin{eqnarray} \label{SP1}
\text{sign}(\hat{\kappa}_{ij}) & = & \text{sign}(w_{ij}^*),  ~\forall j \in S_i~  \\
 \label{SP2}
    \hat{w}_{ij} & = & 0,  ~\forall j \in S_i^c~. 
\end{eqnarray}
Our analysis is built on the {\em primal-dual witness} (PDW) method \cite{W2009}. This method has the following steps. First, only for the sake of analysis, we presuppose that some Oracle provides the true neighbors $S_i$. Therefore, we solve the following problem to recover the signs of true neighbors.
\begin{equation} \label{logistic_S}
\hat{w}_{i,S} = \arg\!\!\min_{(w_{i,S}, \boldsymbol{0}) \in \mathbb{R}^{n-1}}  \ell_i(w_i; D) ~+~ \lambda_{M} ||w_{i, S}||_1~, 
\end{equation}
We then set the components of the dual vector $\kappa_i$ that pertain to neighbors of $i$ to the sign of corresponding components in $\hat{w}_{i,S}$. That is, $\hat{\kappa}_{i, j} ~=~ \text{sign}(\hat{w}_{i,j}), \,\forall j \in S_i$. 
We next set $\hat{w}_{i, S^c} = \boldsymbol{0}$, and thus 
\eqref{SP2} is satisfied. We then solve for $\hat{\kappa}_{i, S^c}$ by plugging $\hat{w}_{i, S}$, $\hat{\kappa}_{i, S}$, and $\hat{w}_{i, S^c}$ in \eqref{KKT}. Thus, we are left to show 
that \eqref{Subgradient} and \eqref{SP1} are satisfied. We impose conditions on $M$, $n$, and $d$ under which these conditions are satisfied with high probability. In fact, we prove a stronger result for \eqref{Subgradient}, namely, strict dual feasibility for non-neighbors, i.e.,  $|\hat{\kappa}_{i, j}| < 1$ for all $j \in S_i^c$.

We argue that our construction yields a unique optimal primal solution $\hat{w}_i$. Specifically, we invoke Lemma 1 from \cite{RWL2010} that states that so long as $||\hat{\kappa}_{i, S^c}||_{\infty} < 1$, any optimal primal solution $\tilde{w}_{i}$ satisfies $\tilde{w}_{i, S^c} = \boldsymbol{0}$. This is established by our construction above. Moreover, this Lemma asserts that $\hat{w}_{i}$ is the unique solution to \eqref{logistic} if $\Lambda_{\min}(\hat{H}_{i, SS}^{M}) > 0$, i.e., if the sample Hessian under $\hat{w}_i$ is positive definite when restricted to the rows and columns in the true support $S_i$.  We show that assumption \eqref{min_eigen} implies $\Lambda_{\min}\left(\hat{H}_{i, SS}^{M}\right) ~\geq~ \alpha^2 \dfrac{C_{min}}{2} > 0$, and this guarantees that we correctly recover the signed neighborhood of $i$.   

To proceed, we define $G_i^M = - \nabla \ell_i(w_i^*; D)$ and rewrite \eqref{KKT} as
\begin{equation} \label{KKT2} 
\nabla \ell_i(\hat{w}_i; D) - \nabla \ell_i(w_i^*; D) ~=~ 
G_i^M ~-~ \lambda_M \hat{\kappa}_i~. \end{equation}
Applying the mean value theorem component-wise, we can write \eqref{KKT2} as
\begin{equation} \label{KKT3}
\nabla^2 \ell_i(w_i^*; D) (\hat{w}_i - w_i^*) ~=~  G_i^M ~-~ \lambda_M \hat{\kappa}_i ~-~ R_i^M~,
\end{equation}
where
$$R_{i, j}^M  ~=~ \left(\nabla^2 \ell_i(\overline{w}_i^{(j)}; D) - \nabla^2 \ell_i(w_i^*; D)\right)_{j}^{\top} (\hat{w}_i - w_i^*)~,$$
for some vector $\overline{w}_i^{(j)} = t_j \hat{w}_i + (1-t_j) w_i^*$, $t_j \in [0, 1]$. Here, $(A)^{\top}_j$ denotes row $j$ of matrix $A$.  

\newtheorem{lemma}{Lemma}
\renewcommand\thelemma{R\arabic{lemma}}

We are now ready to state an important lemma. We will use R1, R2 etc. to indicate that the result is aimed toward provable recovery (i.e. identifiability) of our games.

\begin{lemma} \label{Lemma1}
We have that
\begin{equation*} 
\mathbb{P}\left(||G_{i}^M||_\infty \geq  \dfrac{\lambda_M}{4} \dfrac{\gamma}{2 - \gamma} \right)  ~\leq~ 2 \exp\left(-\dfrac{\gamma^2 \lambda_M^2}{32 \alpha^2 (2 - \gamma)^2} M + \log(n) \right)~,
\end{equation*}
which converges to zero at rate $\exp(-C_{\alpha, \lambda} \lambda_M^2 M)$ (where constant $C_{\alpha, \gamma}$ depends on $\alpha$ and $\gamma$)  whenever 
$$\lambda_{M} \geq \dfrac{8 \alpha (2 - \gamma)}{\gamma} \sqrt{\dfrac{\log(n)}{M}}~.$$
\end{lemma}
\begin{proof}
We note that 
$$G_{i}^{M} ~=~ - \nabla \ell_i(w_i^*; D) ~=~ \dfrac{1}{M}  \sum_{m=1}^M  \underbrace{-\alpha(\sigma_i^{*(m)} - a_i^{*(m)})~ \Phi_{-i}^{(m)}}_{Z_{i, m}}~,$$
where $|Z_{i, m}^u| \leq \alpha$ for each component $Z_{i, u}^m$  of random vector $Z_{i}^m$.  Moreover, $\mathbb{E}(Z_{i, u}^m) = 0$ under $w_i^*$, and  $Z_{i, u}^1, \ldots, Z_{i, u}^M$ are independent. 
Invoking the Hoeffding's inequality, we have that for any $\delta > 0$, 
$$\mathbb{P}(|G_{i, u}^M|  \geq \delta) ~\leq~ 2 \exp\left(- \dfrac{M\delta^2}{2 \alpha^2}\right)~,$$
where  $G_{i, u}^M$ denotes the component at index $u$ of vector $G_{i}^M$.  Setting $\delta = \dfrac{\gamma \lambda_M}{4(2 - \gamma)}$, we get
$$\mathbb{P}\left(|G_{i, u}^M|  \geq  \dfrac{\gamma \lambda_M}{4(2 - \gamma)}  \right) ~\leq~ 2\exp\left(-\dfrac{M}{2\alpha^2} \dfrac{\gamma^2 \lambda_M^2}{16(2-\gamma)^2} \right)~.$$
Then, applying a union bound over indices $u  \in [n-1]$, we get 
\begin{eqnarray*}
\mathbb{P}\left(||G_{i}^M||_{\infty}  \geq  \dfrac{\gamma \lambda_M}{4(2 - \gamma)}  \right) & \leq & 2(n-1)\exp\left(-\dfrac{M}{2\alpha^2} \dfrac{\gamma^2 \lambda_M^2}{16(2-\gamma)^2} \right)\\
& < & 2\exp\left(-\dfrac{M}{2\alpha^2} \dfrac{\gamma^2 \lambda_M^2}{16(2-\gamma)^2} + \log(n) \right)~.\\
\end{eqnarray*}
\end{proof}

\begin{lemma} \label{Lemma2}
Let $\lambda_M d  ~\leq~ \dfrac{\alpha C^2_{\min}}{10 C_{\max}}$ and $||G_{i}^M||_{\infty} ~\leq~ \dfrac{\lambda_M}{4}$. Then,
\begin{equation*}
||\hat{w}_{i, S} - w_{i, S}^*||_2 ~\leq~ \dfrac{5}{\alpha^2 C_{\min}} \lambda_M \sqrt{d}~. 
\end{equation*}
\end{lemma}
\begin{proof}
We define a function $F:  \mathbb{R}^d \to \mathbb{R}$ that quantifies the change in optimization objective at a distance $\Delta_{i, S}$ from the true parameters $w^*_{i, S}$. Specifically,
$$F(\Delta_{i, S})  ~\triangleq~ \ell_i(w_{i,S}^* + \Delta_{i, S}; D) -  \ell_i(w_{i,S}^*; D) + \lambda_M(||w_{i,S}^* + \Delta_{i, S}||_1 - ||w_{i,S}^*||_1)~.$$
Note that $F$ is convex and $F(\boldsymbol{0}) ~=~ 0$.  Moreover, $F$ is minimized for $\hat{\Delta}_{i, S} = \hat{w}_{i, S} - w^*_{i, S}$. Therefore, $F(\hat{\Delta}_{i, S}) \leq 0$.  We show that the function $F$ is strictly positive on the surface of a Euclidean ball of radius 
$B$ for some $B > 0$. Then, the vector $\hat{\Delta}_{i, S}$ lies inside the ball, i.e., $$||\hat{w}_{i, S} - w^*_{i, S}||_2 \leq B~.$$ 
This follows since otherwise, the convex combination $t\hat{\Delta}_{i, S} + (1-t) \boldsymbol{0}$ would lie on boundary of the ball for some $t \in (0, 1)$, which would imply the contradiction
$$F(t\hat{\Delta}_{i, S} + (1-t) \boldsymbol{0})   \leq   t F(\hat{\Delta}_{i, S}) + (1-t) F(\boldsymbol{0}) \leq 0.$$
Therefore, let $\Delta \in \mathbb{R}^d$ be an arbitrary vector such that $||\Delta||_2 = B$.  We then have from Taylor's series
\begin{equation} \label{Taylor} F(\Delta) = \nabla \ell_i(w_{i, S}^*; D)^{\top} \Delta + \Delta^{\top} \nabla^2 \ell(w_{i,S}^* + \theta \Delta; D) \Delta + \lambda_M(||w_{i,S}^* + \Delta||_1 - ||w_{i,S}^*||_1)~, \end{equation}
for some $\theta \in [0, 1]$. We lower bound $F(\Delta)$ by bounding each term on the right side of \eqref{Taylor}. 

We let $B = O \lambda_M \sqrt{d}$ where we will choose $O > 0$ later. From Cauchy-Schwartz inequality, 
\begin{eqnarray} \label{LB1}
\nabla \ell_i(w_{i, S}^*; D)^{\top} \Delta & ~\geq~ & - ||\nabla \ell_i(w_{i, S}^*; D)||_\infty ||\Delta||_1 \\ & ~\geq~ & -||\nabla \ell_i(w_{i, S}^*; D)||_\infty  \sqrt{d} ||\Delta||_2 \\ & ~\geq~ & - (\lambda_M \sqrt{d})^2 \dfrac{O}{4}~,
\end{eqnarray}
where in the last inequality we have used $||\Delta||_2 = B = O \lambda_M \sqrt{d}$, and
 $$- ||\nabla \ell_i(w_{i, S}^*; D)||_\infty  \geq - ||\nabla \ell_i(w_{i}^*; D)||_\infty ~=~  - ||- \nabla \ell_i(w_{i}^*; D)||_\infty  ~=~ -||G_{i}^M||_{\infty} ~\geq~ - \dfrac{\lambda_M}{4}$$
 by our assumption on  $||G_{i}^M||_{\infty}$ in the lemma statement. Next, by triangle inequality, we have
\begin{equation} \label{LB2} \lambda_M(||w_{i,S}^* + \Delta||_1 - ||w_{i,S}^*||_1) ~\geq~ - \lambda_M ||\Delta||_1 ~\geq~   - \lambda_M \sqrt{d} ||\Delta||_2  ~\geq~ - (\lambda_M \sqrt{d})^2 O~.\end{equation}

We now bound the quantity $\Delta^{\top} \nabla^2 \ell(w_{i,S}^* + \theta \Delta; D) \Delta$. We note that 
\begin{eqnarray*}
\Delta^{\top} \nabla^2 \ell(w_{i,S}^* + \theta \Delta; D) \Delta & \geq &  \min_{||\tilde{\Delta}||_2 = B}  \tilde{\Delta}^{\top} \nabla^2 \ell(w_{i,S}^* + \theta \Delta; D) \tilde{\Delta} \\
& \geq & \min_{\tilde{\theta} \in [0, 1]}  B^2 \Lambda_{\min} (\nabla^2 \ell(w_{i,S}^* + \tilde{\theta} \Delta; D))\\
& = & B^2 \min_{\tilde{\theta} \in [0, 1]} \Lambda_{\min}\left(\dfrac{1}{M} \sum_{m=1}^M \eta_i(w_{i,S}^* + \tilde{\theta} \Delta; m)~ \Phi_{-i}^{(m)} \Phi_{-i}^{(m)^{\top}}\right)~.
\end{eqnarray*}
Applying Taylor's series expansion, we note that  $\Delta^{\top} \nabla^2 \ell(w_{i,S}^* + \theta \Delta; D) \Delta$
\begin{eqnarray*}
& \geq & B^2 \Lambda_{\min} \left(\dfrac{1}{M} \sum_{m=1}^M \eta_i(w_{i,S}^*; m) \Phi_{-i}^{(m)} \Phi_{-i}^{(m)^{\top}}\right)\\
& \qquad- & B^2 \max_{\tilde{\theta} \in [0, 1]} \left|\left|\left|\dfrac{1}{M} \sum_{m=1}^M \eta_i'(w_{i,S}^* + \overline{\theta} \Delta; m)(\Phi_{-i}^{(m)^{\top}}\tilde{\theta} \Delta) ~ \Phi_{-i}^{(m)} \Phi_{-i}^{(m)^{\top}}\right|\right|\right|_2~\\
& = & B^2 \Lambda_{\min}(H_{i, SS}^{*M}) - B^2 \max_{\tilde{\theta} \in [0, 1]} \left|\left|\left|\dfrac{1}{M} \sum_{m=1}^M \eta_i'(w_{i,S}^* + \overline{\theta} \Delta; m)(\Phi_{-i}^{(m)^{\top}}\tilde{\theta} \Delta) ~ \Phi_{-i}^{(m)} \Phi_{-i}^{(m)^{\top}}\right|\right|\right|_2~\\
& = & B^2 \alpha^2 C_{\min} - B^2 \max_{\tilde{\theta} \in [0, 1]} \left|\left|\left|\dfrac{1}{M} \sum_{m=1}^M \eta_i'(w_{i,S}^* + \overline{\theta} \Delta; m)(\Phi_{-i}^{(m)^{\top}}\tilde{\theta} \Delta) ~ \Phi_{-i}^{(m)} \Phi_{-i}^{(m)^{\top}}\right|\right|\right|_2~.
\end{eqnarray*}
Now, a simple calculation shows that $|\eta_i'(\cdot)| ~\leq~ \alpha^3$. Moreover, we note for $\tilde{\theta} \in [0, 1]$, $$|\Phi_{-i}^{(m)^{\top}}\tilde{\theta} \Delta| ~\leq~ ||\Phi_{-i}^{(m)}||_\infty ||\tilde{\theta} \Delta||_1 ~\leq~ ||\Phi_{-i}^{(m)}||_\infty ||\Delta||_1 ~\leq~ ||\Delta||_1 ~\leq~ \sqrt{d} ||\Delta||_2 ~=~ O \lambda_M d~.$$
Putting all these facts together, along with our assumption \eqref{max_eigen}, we get 
\begin{equation} \label{LB3}
\Delta^{\top} \nabla^2 \ell(w_{i,S}^* + \theta \Delta; D) \Delta  ~\geq~ B^2 \alpha^2 C_{\min} - B^2 \alpha^3 (O \lambda_M d)C_{\max} ~\geq~  B^2 \alpha^2 \dfrac{C_{\min}}{2}  
\end{equation}
 when $\lambda_M \leq \dfrac{C_{\min}}{2 \alpha C_{\max}Od}$~. Therefore, plugging the lower bounds from \eqref{LB1}, \eqref{LB2}, and \eqref{LB3} in \eqref{Taylor}, 
 $$F(\Delta) ~\geq~ \lambda_M^2 d \left(- \dfrac{O}{4} - O + \dfrac{O^2 \alpha^2 C_{\min}}{2} \right) ~>~ 0~,$$
 for $O = \dfrac{5}{\alpha^2 C_{\min}}$. Thus, for
 $\lambda_M ~\leq~ \dfrac{C_{\min}}{2 \alpha C_{\max}Od} ~=~ \dfrac{\alpha C^2_{\min}}{10 C_{\max}d}$,
 we must have
 $$||\hat{w}_{i, S} - w_{i, S}^*||_2 ~\leq~ B ~=~ O \lambda_{M} \sqrt{d} ~=~  \dfrac{5}{\alpha^2 C_{\min}} \lambda_M \sqrt{d}~.$$
\end{proof}

\begin{lemma}\label{Lemma3}
Let $\lambda_M d ~\leq~ \dfrac{\alpha C_{\min}^2}{100 C_{\max}} \dfrac{\gamma}{2 - \gamma}$ and $||G_{i}^M||_{\infty} ~\leq~ \dfrac{\lambda_M}{4}$. Then, 
\begin{equation*}
\dfrac{||R_{i}^M||_{\infty}}{\lambda_M} ~\leq~ \dfrac{25 C_{\max}}{\alpha C_{\min}^2} \lambda_M d ~\leq~ \dfrac{1}{4} \left(\dfrac{\gamma}{2 - \gamma}\right) ~\leq~ \dfrac{\gamma}{4}~.
\end{equation*}
\end{lemma}
\begin{proof}
We have for $j \in [n]\setminus\{i\}$ and some 
$\overline{w}_i^{(j)} = t_j \hat{w}_i + (1-t_j) w_i^*$, $t_j \in [0, 1]$,
\begin{eqnarray*}
R_{i, j}^M & = & \left(\nabla^2 \ell_i(\overline{w}_i^{(j)}; D) - \nabla^2 \ell_i(w_i^*; D)\right)_{j}^{\top} (\hat{w}_i - w_i^*)  \\
& = & \dfrac{1}{M} \sum_{m=1}^M  \left( \left(\eta_i(\overline{w}_i^{(j)}; m) - \eta_i(w_i^*; m)\right) \Phi_{-i}^{(m)} \Phi_{-i}^{(m)^{\top}}\right)_j^{\top}(\hat{w}_i - w_i^*)  \\ 
& = & \dfrac{1}{M} \sum_{m=1}^M  \left(\eta_i'(\overline{\overline{w}}_i^{(j)}; m) \left(\Phi_{-i}^{(m)^{\top}} (\overline{w}_i^{(j)} - w_i^*)\right) \Phi_{-i}^{(m)} \Phi_{-i}^{(m)^{\top}}\right)_j^{\top}(\hat{w}_i - w_i^*)~,  \\
\end{eqnarray*}
where $\overline{\overline{w}}_i^{(j)}$ is a point on the line between $\overline{w}_i^{(j)}$ and $w_i^*$, by the mean value theorem. We note that $$\left(\Phi_{-i}^{(m)} \Phi_{-i}^{(m)^{\top}}\right)_j^{\top} ~=~ \phi_j^{(m)} \Phi_{-i}^{(m)^{\top}}~.$$
We thus write
\begin{eqnarray*}
R_{i, j}^M & = &  \dfrac{1}{M} \sum_{m=1}^M  \eta_i'(\overline{\overline{w}}_i^{(j)}; m) \phi_{j}^{(m)}  \left((\overline{w}_i^{(j)} - w_i^*)^{\top}\Phi_{-i}^{(m)} \right)  \Phi_{-i}^{(m)^{\top}}(\hat{w}_i - w_i^*)~  \\
& = & \dfrac{1}{M} \sum_{m=1}^M  \eta_i'(\overline{\overline{w}}_i^{(j)}; m) \phi_{j}^{(m)}  \left((\overline{w}_i^{(j)} - w_i^*)^{\top}\Phi_{-i}^{(m)}   \Phi_{-i}^{(m)^{\top}}(\hat{w}_i - w_i^*)\right)  \\
& = & \dfrac{1}{M} \sum_{m=1}^M  \underbrace{\eta_i'(\overline{\overline{w}}_i^{(j)}; m) \phi_{j}^{(m)}}_{p^{(m)}} \underbrace{\left(t_j(\hat{w}_i - w_i^*)^{\top}\Phi_{-i}^{(m)}   \Phi_{-i}^{(m)^{\top}}(\hat{w}_i - w_i^*)\right)}_{q^{(m)}}~,  
\end{eqnarray*}
which is of the form $\dfrac{1}{M}p^{\top}q$, where $p, q \in \mathbb{R}^{M}$. Thus, we have by Cauchy-Schwartz inequality, 
\begin{eqnarray*}
|R_{i, j}^M| & = &  \dfrac{1}{M}|p^{\top}q| ~~\leq~~ \dfrac{1}{M} ||p||_\infty ||q||_1~.
\end{eqnarray*}
It can be shown that $p^{(m)} = \alpha^3 \overline{\overline{\sigma}}_i^{(m)}(1 - \overline{\overline{\sigma}}_i^{(m)})(1-2\overline{\overline{\sigma}}_i^{(m)})$, whereby $||p||_{\infty} \leq \alpha^3$. 

Finally, we see that $q^{(m)} = t_j \left|\left|\Phi_{-i}^{(m)^{\top}}(\hat{w}_i - w_i^*)\right|\right|_2^2 ~\geq~ 0$ since $t_j \in [0, 1]$. Therefore $||q||_1 ~=~ q^{\top} \boldsymbol{1}$, where $\boldsymbol{1} \in \mathbb{R}^M$ is a vector of all ones. Moreover, since $\hat{w}_{i, S^c} = w_{i, S^c}^* = \boldsymbol{0}$, we note that
\begin{eqnarray*}
\dfrac{1}{M}||q||_1 & = & t_j  (\hat{w}_i - w_i^*)^{\top} \left(\dfrac{1}{M}\sum_{m=1}^M \Phi_{-i}^{(m)}   \Phi_{-i}^{(m)^{\top}}\right)(\hat{w}_i - w_i^*)  \\
& = & t_j  (\hat{w}_{i, S} - w_{i, S}^*)^{\top} \left(\dfrac{1}{M}\sum_{m=1}^M \Phi_{-i, S}^{(m)}   \Phi_{-i, S}^{(m)^{\top}}\right)(\hat{w}_{i, S} - w_{i, S}^*)  \\
& \leq & C_{\max} \left|\left|\hat{w}_{i, S} - w_{i, S}^*\right|\right|_2^2~.
\end{eqnarray*}
Since $\gamma \in (0, 1]$, so
$$\lambda_M d ~\leq~ \dfrac{\alpha C_{\min}^2}{100 C_{\max}} \dfrac{\gamma}{2 - \gamma} ~\leq~ \dfrac{\alpha C_{\min}^2}{100 C_{\max}} ~\leq~ \dfrac{\alpha C_{\min}^2}{10 C_{\max}}~.$$
Therefore, we can invoke Lemma \ref{Lemma2} when $||G_i^M||_\infty ~\leq~ \dfrac{\lambda_M}{4}$.  Specifically, we then have for each $j$,  $$|R_{i, j}^M| ~\leq~ \alpha^3 C_{\max} \left|\left|\hat{w}_{i, S} - w_{i, S}^*\right|\right|_2^2 ~\leq~ \alpha^3 C_{\max} \left(\dfrac{5}{\alpha^2 C_{\min}} \lambda_M \sqrt{d}\right)^2 ~=~ \dfrac{25 C_{\max}}{\alpha C_{\min}^2} \lambda_M^2 d~.$$
This immediately yields
$\dfrac{||R_i^M||_\infty}{\lambda_M} ~\leq~  \dfrac{25 C_{\max}}{\alpha C_{\min}^2} \lambda_M d~.$
\end{proof}

\newtheorem{thmI}[lemma]{Theorem}
\renewcommand\thethmI{R\arabic{lemma}}

\begin{thmI}
Let $M > \dfrac{80^2C_{\max}^2}{C_{\min}^4} \left(\dfrac{2 - \gamma}{\gamma}\right)^4 d^2 \log(n)$, and $\lambda_{M} \geq \dfrac{8 \alpha (2 - \gamma)}{\gamma} \sqrt{\dfrac{\log(n)}{M}}$. Suppose the sample satisfies assumptions \eqref{min_eigen}, \eqref{max_eigen}, and \eqref{incoherence}. Consider any player $i \in [n]$. The following results hold with probability at least $1 - 2\exp(-C_{\alpha, \lambda}\lambda_M^2 M) \to 1$ for $i$.
\begin{enumerate}
    \item  The corresponding $\ell_1$-regularized optimization problem has a unique solution, i.e., a unique set of neighbors for $i$. 
    \item The set of predicted neighbors of $i$ is a subset of the true neighbors. Additionally, the predicted set contains all true neighbors $j$ for which $|w_{ij}^*| \geq \dfrac{10}{\alpha^2 C_{\min}} \sqrt{d} \lambda_M$. 
    In particular, the set of true neighbors of $i$ is exactly recovered if 
    $$\min_{j \in S_i} |w_{ij}^*| \geq  \dfrac{10}{\alpha^2 C_{\min}} \sqrt{d} \lambda_M~.$$
\end{enumerate}
Taking a union bound over players, our results imply that we recover the true signed neighborhoods for all players in the LAG with probability at least  $1 - 2n\exp(-C_{\alpha, \lambda}\lambda_M^2 M)$~.
\end{thmI}
\begin{proof}
Since $\lambda_M \geq \dfrac{8 \alpha (2 - \gamma)}{\gamma} \sqrt{\dfrac{\log(n)}{M}}$, Lemma \ref{Lemma1} holds. Thus, with high probability (as stated in the theorem statement), we obtain
\begin{equation} \label{G_M} ||G_i^M||_\infty ~\leq~ \dfrac{\lambda_M}{4} \dfrac{\gamma}{2-\gamma} ~\leq~ \dfrac{\gamma \lambda_M}{4} ~\leq~ \dfrac{\lambda}{4}~,\end{equation}
since $\gamma \in (0, 1]$. Moreover, for the specified lower bound on sample size $M$, a simple computation shows 
\begin{equation} \label{lambda_Md} \lambda_M d ~\leq~ \dfrac{\alpha C^2_{\min}}{10 C_{\max}} \dfrac{\gamma}{2 - \gamma}~.\end{equation}
Thus the conditions required for both Lemma \ref{Lemma2} and Lemma \ref{Lemma3} are satisfied. By our primal-dual construction, $\hat{w}_{i, S^c} = \boldsymbol{0}$. Furthermore, using \eqref{min_eigen}, $\Lambda_{\min}(H_{i, SS}^{*M}) > 0$, and so $H_{i, SS}^{*M}$ is invertible. Separating the rows in the support of $i$ and others, we write \eqref{KKT3} as 
\begin{eqnarray*}
H_{i, S^cS}^{*M}(\hat{w}_{iS} - w^*_{iS}) &=& G_{i, S^c}^M - \lambda_M \hat{\kappa}_{i, S^c} - {R}_{i, S^c}^M \\
H_{i, SS}^{*M}(\hat{w}_{iS} - w^*_{iS}) &=& G_{i, S}^M - \lambda_M \hat{\kappa}_{i, S} - {R}_{i, S}^M~.
\end{eqnarray*}
These two equations can be combined into one as
$$H_{i, S^cS}^{*M}(H_{i, SS}^{*M})^{-1} \left(G_{i, S}^M - \lambda_M \hat{\kappa}_{i, S} - {R}_{i, S}^M \right)  ~=~ G_{i, S^c}^M - \lambda_M \hat{\kappa}_{i, S^c} - {R}_{i, S^c}^M~.$$
Recalling that $||\hat{\kappa}_{i, S}||_\infty < 1$, we immediately get that $\lambda_M||\hat{\kappa}_{i, S^c}||_{\infty}$
\begin{eqnarray*}
 & \leq & \left|\left|\left|H_{i, S^cS}^{*M}(H_{i, SS}^{*M})^{-1}\right|\right|\right|_{\infty} \left(||G_{i,S}^M||_\infty ~+~ ||R_{i, S}^M||_{\infty} + \lambda_M      \right) ~+~ ||G_{i,S^c}^M||_\infty ~+~ ||R_{i, S^c}^M||_{\infty}\\
& \leq & (1 - \gamma) \left(||G_{i,S}^M||_\infty ~+~ ||R_{i, S}^M||_{\infty} + \lambda_M      \right) ~+~ ||G_{i,S^c}^M||_\infty ~+~ ||R_{i, S^c}^M||_{\infty}\\
& \leq & (1 - \gamma) \lambda_M ~+~  ||G_{i}^M||_\infty ~+~ ||R_{i}^M||_{\infty}\\
& \leq & \lambda_M \left(1 - \gamma + \dfrac{\gamma}{4}  + \dfrac{\gamma}{4} \right)\\
& = & \lambda_M \left(1 - \dfrac{\gamma}{2}\right)~.
\end{eqnarray*}
Since $\gamma \in (0, 1]$ and $\lambda_M > 0$, we immediately get $||\hat{\kappa}_{i, S^c}||_{\infty} < 1$. Therefore, strict dual feasibility is established and \eqref{Subgradient} is verified. Then, using Lemma 1 of \cite{RWL2010}, we note that 
any optimal solution $\tilde{w}_i$ of \eqref{logistic} must have $\tilde{w}_{i, S^c} = \boldsymbol{0}$. In particular, we have $\hat{w}_{i, S^c} = \boldsymbol{0}$ as desired. Thus, we can focus on $\hat{w}_{i, S}$. We now prove uniqueness of $\hat{w}_i$ by showing that  $\Lambda_{\min}\left(\hat{H}_{i, SS}^{M}\right) ~>~ 0$. Let $\Delta = \hat{w}_{i, S} - w_{i, S}^* \in \mathbb{R}^d$. Then, using Lemma \ref{Lemma2}, we have
$$||\Delta||_2 ~\leq~ \dfrac{5}{\alpha^2 C_{\min}} \lambda_M \sqrt{d}~.$$
Note that 
\begin{eqnarray*}
\Lambda_{\min}\left(\hat{H}_{i, SS}^{M}\right) & = & \Lambda_{\min}\left(\dfrac{1}{M} \sum_{m=1}^M \eta_i(\hat{w}_{i}; m) ~ \Phi_{-i, S}^{(m)} \Phi_{-i, S}^{(m)^\top}\right)\\
& = & \Lambda_{\min}\left(\dfrac{1}{M} \sum_{m=1}^M \eta_i(\hat{w}_{i, S}; m) ~ \Phi_{-i, S}^{(m)} \Phi_{-i, S}^{(m)^\top}\right)\\
& = & \Lambda_{\min}\left(\dfrac{1}{M} \sum_{m=1}^M \eta_i({w}_{i, S}^* + \Delta; m) ~ \Phi_{-i, S}^{(m)} \Phi_{-i, S}^{(m)^\top}\right)~.\\
\end{eqnarray*}
Performing a Taylor expansion around $w^*_{i, S}$, and making arguments similar to the proof segment between \eqref{LB2} and \eqref{LB3} in Lemma 2, we can show that
\begin{eqnarray*} \Lambda_{\min}\left(\hat{H}_{i, SS}^{M}\right) &\geq& \alpha^2 C_{\min} - \alpha^3 \sqrt{d} ||\Delta||_2 C_{\max} \\
&\geq& \alpha^2 C_{\min} - \left(\dfrac{5 \alpha C_{\max}}{ C_{\min}}\right) \lambda_M d \\
&\geq& \alpha^2 C_{\min} - \alpha^2 \dfrac{C_{\min}}{2} \dfrac{\gamma}{2 - \gamma}\\
&\geq& \alpha^2 \dfrac{C_{\min}}{2}~, \\
\end{eqnarray*}
which is greater than 0. Therefore, $\hat{H}_{i, SS}^{M}$ is positive definite, and Lemma 1 of \cite{RWL2010} guarantees that $\hat{w}_{i}$ is the unique optimal primal solution for \eqref{logistic}. 

We finally argue about the only remaining condition \eqref{SP1}. In order for neighbor $j$ to be correctly recovered with sign, i.e., $\text{sign}(\hat{w}_{ij}) = \text{sign}(w_{ij}^*)$, it suffices to have
\begin{equation} \label{Onlyone} |\hat{w}_{ij} - w_{ij}^*| ~\leq~ \dfrac{|w_{ij}^*|}{2}~. \end{equation}
Moreover to recover the neighborhood of $i$ exactly, it is sufficient to show
\begin{equation}  \label{Sufficient} 
\min_{j \in S_i} |w_{ij}^*| \geq 2 ||\hat{w}_{i, S} - w_{i, S}^*||_\infty~, \end{equation}
which implies \eqref{Onlyone}. 
We note that 
$$||\hat{w}_{i, S} - w_{i, S}^*||_\infty ~\leq~ ||\hat{w}_{i, S} - w_{i, S}^*||_2 ~\leq~ \dfrac{5}{\alpha^2 C_{\min}} \lambda_M \sqrt{d}~.$$ Using \eqref{Sufficient}, it immediately follows that the neighborhood of $i$ is recovered with correct sign if 
$$ \min_{j \in S_i} |w_{ij}^*| \geq  \dfrac{10}{\alpha^2 C_{\min}} \lambda_M \sqrt{d}~.$$
\end{proof}